\documentclass[twoside,11pt]{article}

\usepackage{blindtext}

\usepackage[preprint]{jmlr2e}

\usepackage{amsmath, amssymb, amsfonts, mathtools, nccmath}
\usepackage{amsfonts}
\usepackage{float}
\usepackage{subcaption}
\usepackage{enumitem}
\usepackage{booktabs}
\usepackage{multirow}
\usepackage{arydshln} 
\usepackage{algorithm}
\usepackage{algpseudocode}
\usepackage{tikz}
\usetikzlibrary{arrows.meta,bending,shapes,arrows,positioning,calc,chains,shapes.symbols,decorations.pathmorphing,snakes}
\tikzset{
    vertex/.style={circle, draw, inner sep=2pt, minimum size=20pt}, 
    edge/.style={-latex} 
    }

\def\cum{{\text{cum}}}
\newcommand{\cumeta}[2]{{\omega^{(#1)}_{#2 \dots #2}}} 
\def\pa{{\text{pa}}}
\def\R{{\mathbb{R}}}

\def\comp{{\mathsf{C}}}

\graphicspath{{./Graphics/}}

\usepackage{lastpage}
\jmlrheading{23}{2024}{1-\pageref{LastPage}}{1/21; Revised 5/22}{9/22}{21-0000}{Daniela Schkoda, Elina Robeva and Mathias Drton}

\ShortHeadings{Causal Discovery with Unobserved Confounding}{Schkoda, Robeva and Drton}
\firstpageno{1}

\begin{document}

\title{Causal Discovery of Linear Non-Gaussian Causal Models with Unobserved Confounding}

\author{\name Daniela Schkoda \email daniela.schkoda@tum.de \\
       \addr TUM School of Computation, Information and Technology\\
       Technical University of Munich\\
       85748 Garching bei M\"unchen, Germany
       \AND
       \name Elina Robeva \email erobeva@math.ubc.ca \\
       \addr Department of Mathematics \\University of British Columbia\\
       Vancouver, BC V6T 1Z2,  Canada
       \AND
       \name Mathias Drton \email mathias.drton@tum.de \\
       \addr TUM School of Computation, Information and Technology \& Munich Center for Machine Learning\\
       Technical University of Munich\\
       85748 Garching bei M\"unchen, Germany}

\editor{My editor}

\maketitle

\begin{abstract}
We consider linear non-Gaussian structural equation models that involve latent confounding. In this setting, the causal structure is identifiable, but, in general, it is not possible to identify the specific causal effects. Instead, a finite number of different causal effects result in the same observational distribution. 
Most existing algorithms for identifying these causal effects use overcomplete independent component analysis (ICA), which often suffers from convergence to local optima. Furthermore, the number of latent variables must be known a priori. To address these issues, we propose an algorithm that operates recursively rather than using overcomplete ICA. The algorithm first infers a source, estimates the effect of the source and its latent parents on their descendants, and then eliminates their influence from the data. For both source identification and effect size estimation, we use rank conditions on matrices formed from higher-order cumulants. We prove asymptotic correctness under the mild assumption that locally, the number of latent variables never exceeds the number of observed variables. Simulation studies demonstrate that our method achieves comparable performance to overcomplete ICA even though it does not know the number of latents in advance.
\end{abstract}


\section{Introduction}
Linear non-Gaussian acyclic models are a powerful framework for causal inference \citep{Shimizu2022}. Their non-Gaussianity renders the causal structure identifiable; consequently, the models form the foundation for many algorithms for causal discovery. However, when some of the variables are unobserved, inference becomes more involved due to possible latent confounding. While the topological order can still be uniquely determined, the causal effects generally cannot. This paper proposes a recursive approach that can deduce the causal structure and all possible causal effects based solely on observational data and allowing for the possibility of latent confounding.

We denote the observed variables by $X = (X_1, \ldots, X_p)$ and the unobserved variables by $L = (L_1, \dots, L_\ell)$. The causal structure is represented by a directed acyclic graph (DAG) $\mathcal{G} = (V, E)$, where each vertex in the vertex set $V = \{X_1, \ldots, X_p\} \cup \{L_1, \ldots, L_\ell\}$ corresponds to one of the variables and the edges in $E \subseteq V \times V$ represent direct causal effects. The latent variable  $\{L_1,\dots,L_\ell\}$ are assumed to be independent latent factors represented as source nodes in $\mathcal{G}$; compare, e.g., \cite{BarberDSW2022} or the discussion of canonical models in \cite{Salehkaleybar2020}. The linear non-Gaussian acyclic model then postulates that \begin{align}\label{def:lingam_latent_conf}
X_i = \sum_{X_j \in \text{pa}(i)} \lambda_{ij} X_j + \sum_{L_j \in \text{pa}(i)}  \gamma_{ij} L_{j}  + \epsilon_i, \quad (i=1,\ldots,p),
\end{align} 
where all $\epsilon_j$ and $L_j$ are mutually independent and non-Gaussian, and $\text{pa}(i)=\{j\in V:(j,i)\in E\}$ is the set of parents of vertex $i$ in the graph $\mathcal{G}$. The coefficients $\lambda_{ij}\in\mathbb{R}$ are parameters quantifying the direct causal effects among the observed variables, and the $\gamma_{ij}\in\mathbb{R}$ similarly constitute direct effects originated from latent variables. We emphasize that the graph is assumed to be acyclic and the system in \eqref{def:lingam_latent_conf} has a unique solution $X$ for a given choice of $L$, $\epsilon$, $\lambda_{ij}$ and $\gamma_{ij}$.
\subsection{Related Work}
For the case $\ell=0$, where all variables are observed, \cite{Shimizu2006} show that the causal structure, as well as all causal effects, are identifiable in the sense that there is a unique DAG and a unique choice of edge weights $\lambda_{ij}$ that lead to the observed distribution. Moreover, they propose the method ICA-LiNGAM to estimate both. Its idea is to rewrite \eqref{def:lingam_latent_conf} as $X = B \epsilon$ for $B = (I-\Lambda)^{-1}$ and then estimate the matrices $B$ and $\epsilon$ using an independent component analysis algorithm. We refer to the recent account of \cite{Auddy2023} for more details on ICA. Since most ICA algorithms use gradient descent, there are no guarantees that the algorithm will indeed converge to the true solution. Furthermore, the method is not scale invariant \citep{Shimizu2011}.  To improve on these problems, \cite{Shimizu2011} propose the alternative method DirectLiNGAM, which recursively identifies a source node and estimates the causal effects using regression and independence tests. The effect of the source is then removed from the data, and the procedure continues until all nodes are identified; see also \cite{WangD2020}, where the approach is customized to sparse high-dimensional settings. 

The ideas behind the two LiNGAM algorithms are also the basis for many algorithms for problems with an arbitrary number of latent variables $\ell$. The methods of \cite{Hoyer2008} and
\cite{Salehkaleybar2020} use overcomplete ICA. To this end, they rewrite equation \eqref{def:lingam_latent_conf} as
\begin{align*}X=B \eta,\end{align*}
where $\eta = (\epsilon, L)$ is the vector of all exogenous sources, and we define the path matrix $B = (I_p - \Lambda)^{-1} \begin{pmatrix}
I_p & \Gamma \end{pmatrix}$. While the approach by \cite{Hoyer2008} can only estimate the causal effects between pairs of variables with no common confounders. \cite{Salehkaleybar2020} does not impose any assumption on the graph and finds the causal order and all causal effects compatible with the data. However, in practice, the overcomplete ICA method is far more susceptible to both optimization and statistical errors than the standard ICA approach in the fully observed case. 

Therefore, other methods work in the vein of DirectLiNGAM. Similar to DirectLiNGAM, the methods IvLiNGAM \citep{Entner2010}, ParcelLiNGAM \citep{Tashiro2014}, RCD \citep{Maeda2020}, and BANG \citep{Wang2023} all rely on independence tests of certain residuals. IvLiNGAM can find the order and effect sizes in all subsets unaffected by confounding. ParcelLiNGAM can fully discover the structure for all ancestral graphs, meaning there is no confounding between each observed variable and any of its ancestors. RCD works for arbitrary graphs, but for confounded pairs of observed variables, they do not detect the causal direction. 
BANG can find the causal order and all causal effects for all DAGs where no parent-child pair is affected by confounding. In contrast to all the other methods, BANG allows for non-linear confounding.

The method most similar to the approach we propose in this paper is that of \cite{Cai2023}. Like DirectLiNGAM, it works recursively, but unlike DirectLiNGAM, it does not test for independence of residuals. Instead, using conditions involving cumulants, \cite{Cai2023} find structures of the form $X_i \leftarrow L \rightarrow X_j$ or $X_i \to X_j$, estimate the coefficients in these structures, remove the effect of the latent, and continue. For this to work, one of these two structures must exist in each iteration. So, they have to make assumptions about the true graph, including that each latent has at least three observed children, amongst which one child is unaffected by any other latents. In related work, \cite{Chen2024} focus on the case of two observed variables and one latent variable and propose a method to infer the direction and the effect size between the two observed variables.
\subsection{Contribution}
In this paper, we extend the approaches just mentioned and prove conditions for finding a source in a completely arbitrary DAG. Once a source is found, we estimate its effects on its descendants using polynomial equations for the edge weights. Then, its effect is removed from the data, and the procedure continues until the entire structure is discovered. Our method finds the whole structure whenever locally the number of observed nodes is higher than the number of latents, as made precise in Lemma \ref{lem:regression_possible}.  Beyond its primary application, the algorithm can be used for causal effect identification, where one is interested in a single causal effect. While \cite{Tramontano2024} provides graphical criteria to decide whether a specific effect is identifiable, we are, to our knowledge, the first to present a formula for these effects in terms of the cumulants rather than using overcomplete ICA.
Alongside our proposed method, we prove a formula for the number of choices for edge weights that lead to the same observed distribution.

\subsection{Notation} Let $i,j \in \mathbb{N}$. With $[i]=\{1, \dots, i\}$ we refer to the set of all natural numbers up to $i$. The $i$th basis vector is denoted by $e_i \in \mathbb{R}^p$. Given a matrix $A \in \mathbb{R}^{p \times q}$,  $A_{:,j}$ stands for its $j$th column, $A_{i:,:}$ for its submatrix selecting all rows starting from row $i$, $A_{:i,:}$ for its submatrix selecting all rows up to and including row $i$, and $A_{-1,:}$ for the submatrix with the last row removed, similarly for column selection. We write $I_p$ for the identity matrix in $\mathbb{R}^{p \times p}$ and the Kronecker delta function is given by
\begin{equation*}
    \delta_{ij} = \begin{cases}
        1 \text{ if }i=j,\\
        0 \text{ otherwise.}
    \end{cases}
\end{equation*}

\section{Preliminaries}

In this section, we provide background and preliminary results on linear structural equation models with latent variables as well as on higher-order cumulant tensors.
Throughout, we use standard terminology in graphical modeling; see, e.g., \citet[Part I]{handbook}.

\subsection{Linear Structural Equation Models}\label{sec:lin_sem}

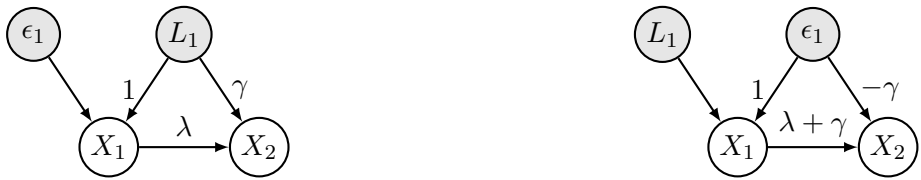
\begin{figure}
  \centering
  \begin{subfigure}[b]{0.4\textwidth}
    \centering
    \begin{tikzpicture}[thick]
      \node[vertex] (1) at (-1,0) {$X_1$};
      \node[vertex] (2) at (1,0) {$X_2$};
      \node[vertex, fill=gray!20] (3) at (0,1.5) {$L_1$};
      \node[vertex, fill=gray!20] (4) at (-2,1.5) {$\epsilon_1$};
      \draw[edge] (1) -- node[above] {$\lambda$} (2);
      \draw[edge] (3) -- node[left] {$1$} (1);
      \draw[edge] (4) -- (1);
      \draw[edge] (3) -- node[right] {$\gamma$} (2); 
    \end{tikzpicture}
  \end{subfigure}
  \hspace{2cm}
  \begin{subfigure}[b]{0.4\textwidth}
    \centering
    \begin{tikzpicture}[thick]
      \node[vertex] (1) at (-1,0) {$X_1$};
      \node[vertex] (2) at (1,0) {$X_2$};
      \node[vertex, fill=gray!20] (3) at (0,1.5) {$\epsilon_1$};
      \node[vertex, fill=gray!20] (4) at (-2,1.5) {$L_1$};
      \draw[edge] (1) -- node[above] {$\lambda+\gamma$} (2);
      \draw[edge] (3) -- node[left] {$1$} (1);
      \draw[edge] (4) -- (1);
      \draw[edge] (3) -- node[right] {$-\gamma$} (2); 
    \end{tikzpicture}
  \end{subfigure}
  \caption{Two parameter sets yielding the same observed distribution.}
  \label{fig:non-identifiability}
\end{figure} 

Due to \cite{Hoyer2008}, every linear structural equation model can be transformed in a way that each latent has no parents and at least two children,  while leaving the observed distribution as well as the total causal effects among observed variables the same. So, we restrict ourselves to this case. When clear from context, we may use the shorthand $v$ instead of $X_v$. For an  observed node $v$, we partition the set of its parents into $\pa(v)=\text{pa}_{\text{o}}(v)\cup \text{pa}_{\text{l}}(v)$, where $\text{pa}_{\text{o}}(v)=\pa(v) \cap \{X_1,\dots,X_p\}$ is the set of its observed parents and $\text{pa}_{\text{l}}(v)=\pa(v) \cap \{L_1,\dots,L_\ell\}$ the set of its latent parents.
Then, the linear  structural equation model $\mathcal{M}(\mathcal{G})$ for the graph $\mathcal{G}$ is the set of all joint distributions $P^X$ of observed random vectors $X=(X_1, \dots, X_p)$ solving the structural equations
\begin{align}\label{eq:introduction_latent_conf}
X_i = \sum_{X_j\in \text{pa}_{\text{o}}(i)} \lambda_{ij} X_j + \sum_{L_j \in \text{pa}_{\text{l}}(i)}  \gamma_{ij} L_j  + \epsilon_i, \quad (i=1,\ldots,p),
\end{align} for a choice of real coefficients $\lambda_{ij}$ and $\gamma_{ij}$ and random vectors $\epsilon = (\epsilon_1, \ldots, \epsilon_p)$ and $L = (L_1, \ldots, L_\ell)$ that are mutually independent, with independent and non-Gaussian components. 

In the sequel, we assume that all moments of $\epsilon$ and $L$ up to some order $k$ are finite. Without loss of generality, we assume that random vectors $\epsilon$ and $L$ are centered. We can arbitrarily rescale each latent variable $L_j$ by some $\alpha_j \neq 0$ and rescale all $\gamma_{ij}$ by $\alpha_j^{-1}$ at the same time without changing $P^X$. To fix the scale, for each $j=p+1, \dots, p+\ell$, we set $\gamma_{ij}=1$ for $i$ an oldest child among all children of $L_j$. 
Moreover, we can always relabel the latents without changing the observed distribution. Therefore, when discussing identifiability, we always mean identifiability up to permuting the latents and choosing a different oldest child $i$ to set $\gamma_{ij}=1$. Lastly, throughout the paper, we assume that all non-zero coefficients in \eqref{eq:introduction_latent_conf}, as well as the cumulants of  $\eta=(\epsilon,L)$, are generic (in particular, our results hold for Lebesgue almost every choice of coefficients and cumulants).

By collecting the causal effects in the matrices $\Lambda = (\lambda_{ij}) \in \mathbb{R}^{p \times p}$ and  $\Gamma = (\gamma_{ij}) \in \mathbb{R}^{p \times \ell}$, we can rewrite \eqref{eq:introduction_latent_conf} as
\begin{align*}
X = \Lambda X + \Gamma L + \epsilon
\end{align*}
or, equivalently,
$$X = B \eta$$
for  $\eta = (\epsilon, L)$ the vector of all exogenous sources and the so-called path matrix $B = (I_p - \Lambda)^{-1} \begin{pmatrix}
  I_p & \Gamma
\end{pmatrix}$. Each entry $b_{ij}$ of $B$ encodes the total causal effect from the exogenous source $\eta_j$ to $X_i$. The matrix $B$ is in one-to-one correspondence with the pair $(\Lambda, \Gamma)$ since they can be recovered as
\begin{align}\label{eq:recover_Lambda_Gamma}
  \Lambda = I_p - \left(B_{:, :p}\right)^{-1}, \quad \Gamma=(I_p-\Lambda)B_{:, (p+1):}.
  \end{align}
Note that the genericity assumption we make ensures faithfulness, meaning that whenever there is path from $i$ to $j$, $b_{ji}\neq 0$. Intuitively, this means that the causal effects from $\eta_j$ going through different paths to $X_i$ are not canceled out. 
Under this faithfulness assumption, \cite{Salehkaleybar2020} showed that the causal order, as well as the number of latents, is uniquely identifiable from the distribution. Moreover, they proved that the number of choices for $B$ compatible with $X$ in the sense that there exists some $\eta$ with independent components and $X=B\eta$ is given by
\begin{align*}
n_\mathcal{G} = \prod_{v \text{ observed node}} |\text{exog}(v)| + 1,
\end{align*} where 
\[
\text{exog}(v) = \{L_j \in \text{pa}_{\text{l}}(v): L_j \text{ has the same observed descendants as }v\}.
\]

\begin{example}
    The graph depicted in Figure \ref{fig:non-identifiability} has $n_\mathcal{G}=2$. The two feasible choices of parameters giving the same distribution are the following: If
\begin{align*}
X = \Lambda X + \Gamma L + \epsilon
\end{align*}
then choosing
\begin{align*}
\lambda' &= \lambda + \gamma , &
 \gamma' &= -\gamma, \\
L_1' &= \epsilon_1, &
\epsilon_1' &= L_1
\end{align*}
 gives the same observed distribution since
 \begin{alignat*}{3}
 X_1 &= L_1 + \epsilon_1 & &= L_1' + \epsilon_1',\\
 X_2 &= \lambda X_1 + \gamma L_1 + \epsilon_2 &  &= \lambda' X_1 + \gamma' L_1' + \epsilon_2.
 \end{alignat*}
Intuitively, the non-identifiability corresponds to swapping the roles of $L_1$ and $\epsilon_1$, the two exogenous sources pointing to $X_1$.
\end{example}
    
The example just given generalizes. Each of the $n_\mathcal{G}$ choices can be obtained by swapping the $v$th and $(j+p)$th  columns in $B$, where $v$ is some observed node and $L_j \in \text{exog}(v)$. At the same time, the corresponding elements in $\eta$ need to be swapped. The reason that precisely such pairs of columns can be swapped can be seen when looking at the sparsity pattern of $B$: the $v$th column of $B$ always needs to have zeros for all non-descendants of $v$. Apart from the original $v$th column, the only other columns with this property are the columns for the latents $L_j \in \text{exog}(v)$.

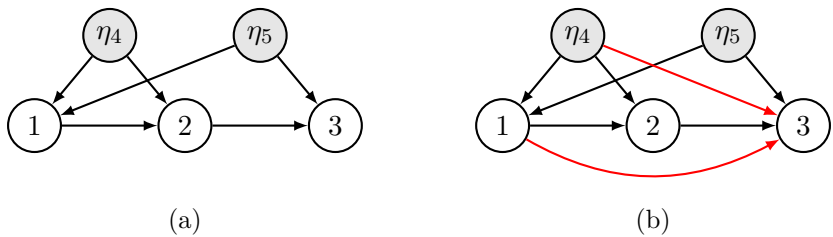
\begin{figure}
  \centering
  \begin{subfigure}[b]{0.4\textwidth}
    \centering
    \begin{tikzpicture}[thick, baseline=(L1).base]
              \node[vertex] (1) at (-2,0) {$1$}; 
              \node[vertex, fill=gray!20] (L1) at (-1,1.2) {$\eta_4$};
              \node[vertex] (3) at (2,0) {$3$};
              \node[vertex] (2) at (0,0) {$2$};
              \node[vertex, fill=gray!20] (L2) at (1,1.2) {$\eta_5$};
              \draw[edge, white] (1) to[bend right] (3);
              \draw[edge] (1) -- (2);
              \draw[edge] (L1) -- (2);
              \draw[edge] (L1) -- (1);
              \draw[edge] (L2) -- (1);
              \draw[edge] (L2) -- (3);
              \draw[edge] (2) -- (3);
          \end{tikzpicture}
          \subcaption{ }\label{fig:Additional_edges_original_graph}
  \end{subfigure}
  \begin{subfigure}[b]{0.4\textwidth}
    \centering
    \begin{tikzpicture}[thick, baseline=(L1).base]
              \node[vertex] (1) at (-2,0) {$1$}; 
              \node[vertex, fill=gray!20] (L1) at (-1,1.2) {$\eta_4$};
              \node[vertex] (3) at (2,0) {$3$};
              \node[vertex] (2) at (0,0) {$2$};
              \node[vertex, fill=gray!20] (L2) at (1,1.2) {$\eta_5$};
              \draw[edge] (1) -- (2);
              \draw[edge, red] (1) to[bend right] (3);
              \draw[edge] (L1) -- (2);
              \draw[edge] (L1) -- (1);
              \draw[edge] (L2) -- (1);
              \draw[edge, red] (L1) -- (3);
              \draw[edge] (L2) -- (3);
              \draw[edge] (2) -- (3);
          \end{tikzpicture}
          \subcaption{ }\label{fig:Additional_edges_denser_graph}
  \end{subfigure}
          \caption{New edges introduced by swapping exogenous sources.}\label{fig:Additional_edges}
  \end{figure}

However, while these $n_\mathcal{G}$ possible path matrices have the same sparsity pattern, they do not always yield parameters $(\Lambda, \Gamma)$ having the same sparsity pattern. For example, the path matrix for the graph in Figure \ref{fig:Additional_edges_original_graph} is
  \begin{align*}
    B = \begin{pmatrix}1 & 0 & 0 & 1 & 1\\\lambda_{21} & 1 & 0 & \lambda_{21}+\gamma_{21} & \lambda_{21}\\
    \lambda_{21} \lambda_{32} & \lambda_{32} & 1 & \gamma_{21} \lambda_{32} + \lambda_{21} \lambda_{32} & \gamma_{32} + \lambda_{21} \lambda_{32}\end{pmatrix}.
    \end{align*}
Swapping the first and last column to obtain another compatible path matrix $B'$ and passing back to $\Gamma', \Lambda'$, we obtain
\begin{align*}
  \Lambda' &=  \begin{pmatrix}0 & 0 & 0\\\lambda_{21} & 0 & 0\\\gamma_{32} & \lambda_{32} & 0\end{pmatrix}, 
  & \Gamma' &= \begin{pmatrix}1 & 1\\\gamma_{21} & 0\\- \gamma_{32} & - \gamma_{32}\end{pmatrix}
\end{align*}
belonging to the denser graph in Figure \ref{fig:Additional_edges_denser_graph}.
Since our main interest lies in finding all compatible $(\Lambda, \Gamma)$ that are as sparse as possible, we want to examine how many such choices exist.
Denote by $\text{sib}(v)$ the set of all nodes who share at least one common latent or observed parent with $v$.

\begin{lemma}\label{lem:number_of_sparsest_choices} If $P^X \in \mathcal{M}_\mathcal{G}$ is defined via generic coefficients, then $\mathcal{G}$ is the unique minimal graph such that $P^X \in \mathcal{M}_\mathcal{G}$ and the number of choices for $(\Lambda, \Gamma)$ compatible with $P^X$ and $\mathcal{G}$ is
  \begin{align*}
    n_\mathcal{G, \text{sparse}}=\prod_{v \text{ observed node}} |\{L \in \text{exog}(v): \text{sib}(v)\subseteq \text{ch}(L)\subseteq \text{ch}(v) \cup \{v\}\}| + 1.
    \end{align*}
\end{lemma}
This result is proven in the appendix by checking which additional edges are introduced by swapping columns in the path matrix.
As a special case, the lemma answers the question of when the edge weights are uniquely identifiable, a question also explored in \cite{Adams2021}. They arrive at necessary conditions for unique identifiability of the edge weights, which are more restrictive than ours. Still, our findings do not conflict with each other since \cite{Adams2021} allow latents to have arbitrarily many parents and children.

Prior to presenting the theoretical foundations of our method, we establish some terminology on tensors and cumulants, which play a crucial role in our estimation method.
\subsection{Tensors and Cumulants}
By $(\mathbb{R}^m)^{\otimes k}$, we denote the $k$-fold tensor product of $\mathbb{R}^m$, and by
$$\text{Sym}_k(\mathbb{R}^m) = \{T \in \left(\mathbb{R}^m\right)^{\otimes k}: t_{i_1\ldots i_k} = t_{i_{\pi(1)}\ldots i_{\pi(k)}} \text{ for all
permutations }\pi: [k] \to [k]\},$$ the subspace of symmetric tensors. In our discussion, all considered tensors are cumulant tensors. For a $m$-variate random vector $Z$ with joint distribution $P^Z$, the
$k$th-order cumulant tensor of $P^Z$  is the tensor  $\cum^{(k)}(P^{Z}) \in
\text{Sym}_k(\mathbb{R}^m)$ 
given by
\begin{equation*}
\Big(\cum^{(k)}\big(P^{Z}\big)\Big)_{i_1 \ldots i_k} = \sum_{(I_1, \ldots, I_h)} (-1)^{h-1} (h-1)! E \left( \prod_{j \in I_1} Z_j \right) \cdots E \left( \prod_{j \in I_h} Z_j \right),
\end{equation*}
where $(I_1, \ldots, I_h)$ is an arbitrary partition of $(i_1, \ldots,
i_k)$. If $Z$ is centered, the second- and third-order cumulant tensors are the same as the second- and third-order moment tensors, respectively. For higher orders, the moment and cumulant tensors generally differ. Importantly, a cumulant tensor of a random vector with independent components is diagonal. Furthermore, taking cumulants commutes with summation if the summands are stochastically independent. Those two properties are attractive when working with linear structural equation models. Specifically, denoting the $k$th-order cumulant of $P^X$ by $C^{(k)}$ and the $k$th-order cumulant of $P^\eta$ by $\Omega^{(k)}$, the following parametrization shown in \cite{Comon_2010} holds. 
\begin{lemma}\label{lem:parametrization} 
If $P^X$ satisfies a linear structural equation model with $\ell$ latent confounders, then 
\begin{equation*}
c^{(k)}_{i_1 \ldots i_k} = \sum_{j = 1}^{p+\ell} \cumeta{k}{j} b_{i_1 j} \cdots b_{i_k j}, \quad (i_1, \ldots, i_k \in [p]).
\end{equation*}
\end{lemma} In the following, when referring to single entries $c^{(k)}_{i_1 \ldots i_k}$ or $\cumeta{k}{j}$ for low orders $k\leq 4$, we might omit the superscript since the order is clear from the number of indices.
\section{Two Observed Variables}
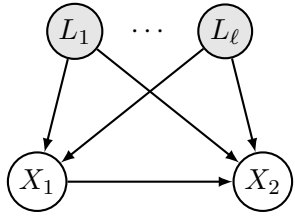
\begin{figure}
  \centering
  \begin{tikzpicture}[thick]
\node[vertex] (1) at (-1.5,0) {${X_1}$};
\node[vertex] (2) at (1.5,0) {${X_2}$};
\node[vertex, fill=gray!20] (3) at (-1,2) {${L_1}$};
\node (5) at (0,2) {$\cdots$};
\node[vertex, fill=gray!20] (4) at (1,2) {${L_\ell}$};
\draw[edge] (1) -- (2);
\draw[edge] (3) -- (1);
\draw[edge] (3) -- (2); 
\draw[edge] (4) -- (1); 
\draw[edge] (4) -- (2); 
  \end{tikzpicture}\\
  \caption{Graph $\mathcal{G}_{2,\ell}$.}\label{fig:2_nodes}
\end{figure}

We first focus on the case of two observed variables, which provides the foundation for our algorithm for an arbitrary number of variables. More specifically, we analyze the linear structural equation model $\mathcal{M}_{2, \ell}$ for the graph $\mathcal{G}_{2,\ell}$ depicted in Figure \ref{fig:2_nodes}. This encompasses all possible graphs for $p=2$ since we allow the edge weight $\lambda_{21}$ to be zero. In contrast, all $\gamma_{ij} \neq 0$ because we assume that each latent has at least two children. 
For the graph $\mathcal{G}_{2,\ell}$, there are $n_{\mathcal{G}_{2,\ell}} = \ell+1$ compatible path matrices. If $X = \Lambda X + \Gamma L + \epsilon$ is one feasible choice, then the other $\ell$ options arise by swapping $\epsilon_1$ and $L_j$ for some $j \in [\ell]$, namely,
\begin{align*}
  \epsilon_1' &= L_j,\\
  L_j' &= \epsilon_1,\\
\lambda_{21}' &= \lambda_{21} + \gamma_{2j}, \\
\gamma_{2j}' &= -\gamma_{2j} ,\\
 \gamma_{2i}' &= \gamma_{2i}-\gamma_{2j} \quad (i \neq j).
\end{align*}
If $\lambda_{21}\neq 0$, then $n_{\mathcal{G}_{2,\ell}, \text{sparse}}$ coincides with $n_{\mathcal{G}_{2,\ell}}$. Otherwise, $n_{\mathcal{G}_{{2,\ell}}, \text{sparse}} = 1$.
We first address how to infer the number of latent variables and the causal order.

\subsection{Distinguish Cause and Effect} A matrix formed from cumulants gives us a condition to recover the number of latents $\ell$, as well as the source.  
\begin{example}\label{ex:contraints_topological_order_p=2}
    For example, if there are no latents, consider the two matrices $$A^{(2,3)}_{1 \to 2} = \begin{pmatrix}
c_{11} & c_{12} \\      
c_{111} & c_{112}\\
c_{112} & c_{122}
\end{pmatrix}, \quad A^{(2,3)}_{2 \to 1} = \begin{pmatrix}
c_{22} & c_{12} \\      
c_{222} & c_{122}\\
c_{122} & c_{112}
\end{pmatrix}.$$ 
If $1$ is the source, then the left matrix, but not the right matrix, has rank $1$, and vice-versa if $2$ is the source. Similarly, for $\ell=1$, we define $$A^{(3,4)}_{1 \to 2} = \begin{pmatrix}
c_{111} & c_{112} & c_{122} \\
c_{1111} & c_{1112} & c_{1122} \\
c_{1112} & c_{1122} & c_{1222}\\\end{pmatrix}, \quad A^{(3,4)}_{2 \to 1} = \begin{pmatrix}
c_{222} & c_{122} & c_{112} \\
c_{2222} & c_{1222} & c_{1122} \\
c_{1222} & c_{1122} & c_{1112}\\\end{pmatrix}.$$ Then, $A^{(3,4)}_{1 \to 2}$ has rank $2$ if $1$ is the source, and $A^{(3,4)}_{2 \to 1}$ has rank 1 if $2$ is the source.
\end{example}

To extend to an arbitrary number of latents, we introduce the notion of a symmetric flattening. For $h \leq k$,  the $h$th flattening $\text{fl}_h(T)$ of the  symmetric tensor $T \in \text{Sym}_k(\mathbb{R}^m)$ is the $\binom{m+h-1}{k-h} \times \binom{m+h-1}{h}$ matrix with rows indexed by $(i_{h+1}, \ldots, i_k) \in [m]^{k-h}$ with $i_{h+1} \leq \cdots \le i_k$, columns indexed by $(i_1, \ldots, i_h) \in [m]^{h}$ with $i_1 \leq \cdots \leq i_h$, and entries given by
$$\left(\text{fl}_h(T)\right)_{(i_{h+1},\ldots,i_k),(i_1,\ldots, i_h)} = t_{i_1\ldots i_k}.$$ 
For $k_1 < k_2$, the matrix $A^{(k_1, \ldots, k_2)}_{1 \to 2}\in \mathbb{R}^{1 + \dots + (k_2-k_1+1)\times k_1}$ is constructed by stacking the $k_1$th symmetric flattenings of $C^{(k_1)}, \dots, C^{(k_2)}$ vertically and then removing the last column, namely,
\begin{align*}
A^{(k_1, \ldots, k_2)}_{1 \to 2} 
= \left(\begin{array}{c}
\quad \text{fl}_{k_1}\left(C^{(k_1)}\right)\quad \\[3pt] \hdashline
\vdots \vspace{1pt}
\\ \hdashline \\ \quad \text{fl}_{k_1}\left(C^{(k_2)}\right) \quad 
\end{array}\right)_{:,-1} 
= \left(\begin{array}{cccc}
c^{(k_1)}_{11\dots 11} & c^{(k_1)}_{11\dots 12} & \cdots & c^{(k_1)}_{12\dots 22} \vspace{1pt} \\[3pt] \hdashline \\[-8pt]
c^{(k_1+1)}_{111\dots 11} & c^{(k_1+1)}_{111\dots 12} & \cdots & c^{(k_1+1)}_{112\dots 22}  \\[3pt]
c^{(k_1+1)}_{211\dots 11} & c^{(k_1+1)}_{211\dots 12} & \cdots & c^{(k_1+1)}_{212\dots 22}   \\[3pt]
\hdashline \vspace{3pt}
\vdots & \vdots & \ddots & \vdots \vspace{1pt}
\\ \hdashline \\[-8pt]
c^{(k_2)}_{1\dots111\dots 11} & c^{(k_2)}_{1\dots111\dots 12} & \cdots & c^{(k_2)}_{1\dots112\dots 22} \\
\vdots & \vdots & \ddots & \vdots \\
c^{(k_2)}_{2\dots211\dots 11} & c^{(k_2)}_{2\dots211\dots 12} & \cdots & c^{(k_2)}_{2\dots212\dots 22}
\end{array}\right).
\end{align*}
Writing  $m$ for the minimum of the number of rows and columns of $A^{(k_1, \ldots, k_2)}_{1 \to 2}$, we obtain the following result, which is proven in the appendix.
\begin{theorem}\label{thm:rank_condition} If $P^X \in \mathcal{M}_{2, \ell}$ and $1$ is a source, then
  \begin{enumerate}[label=\alph*), topsep=0pt, itemsep=0pt]
    \item $A^{(k_1, \ldots, k_2)}_{1 \to 2}$ has rank at most $\ell+1$, and, generically exactly rank $\min(\ell+1, m)$.
    \item If $\lambda_{21} \neq 0$, $A^{(k_1, \ldots, k_2)}_{2 \to 1}$ has rank at most $\ell+2$,  and generically exactly rank $\min(\ell+2, m)$.
    \item If $\lambda_{21} = 0$, $A^{(k_1, \ldots, k_2)}_{2 \to 1}$ has rank at most $\ell+1$,  and generically exactly rank $\min(\ell+1, m)$.
  \end{enumerate}
\end{theorem}
While the theorem holds true for arbitrary choices of orders $k_1, k_2$, in practice, we do not want to use higher orders than necessary. To obtain a non-trivial rank bound, the smallest possible choice of $k_1$ is $\ell+2$ otherwise there would be too few columns. For the number of rows to be large enough, we need
\begin{align*}
\ell+1 &< 1 + 2 + \dots + (k_2-k_1+1).
\end{align*} Since $\ell+1$ is an integer, this is equivalent to
\begin{align*}  
\ell+2 &\leq 1 + 2 + \dots + (k_2-k_1+1) = \frac{(k_2-k_1+1)(k_2-k_1+2)}{2}. 
\end{align*}
So,  $x=k_2-k_1$ fulfills
\begin{align*}
0 \leq \frac{1}{2}(x^2+3x+2-2 \ell-4) = \frac{1}{2}(x^2+3x-2(\ell+1)),
\end{align*}
yielding that $k_2-k_1$ needs to be greater than or equal to the maximal root of the quadratic polynomial on the right-hand side, which is $\frac{1}{2}(-3+\sqrt{8 \ell+17})$.
%
For example for $\ell=0,1$, this choice of orders results in $(k_1, k_2) = (2,3)$ and $(k_1, k_2) = (3,4)$, respectively, as in Example \ref{ex:contraints_topological_order_p=2}.  Henceforth, we denote $A^{(\ell)}_{1 \to 2} = A^{(k_1, k_2)}_{1 \to 2}$ where $(k_1, k_2)$ is the smallest possible choice yielding a non-trivial constraint.

The number of latents is unknown a priori, but we can derive it by sequentially testing the rank condition for increasing $\ell$.

\subsection{Estimating Effect Sizes}\label{sec:equations_for_lambda}
As soon as the source and the number of latents are inferred, the next natural step is to estimate the causal effects. Recall that $b_{21}$ can only be identified up to permuting it with $b_{23}, \dots, b_{2,\ell+1}$, which encompass all remaining causal effects in the model. Thus, by finding all possible choices for $b_{21}$, we instantaneously find all entries of $B$.

Our estimation procedure builds on the rank condition from Theorem \ref{thm:rank_condition}. Specifically, we extend the matrix $A^{(\ell)}_{1 \to 2}$ by adding
\begin{equation*}
\begin{pmatrix}
1 & b_{21} & \dots & b_{21}^{\ell+2}
\end{pmatrix}
\end{equation*} 
as an additional row on top and denote the result by $\tilde{A}_{1 \to 2}$. This extension does not increase the rank of this matrix; therefore, its minors provide us with polynomial equations for $b_{21}$. 
\begin{theorem}\label{thm:effect_sizes} Consider the determinant of an $\ell+2 \times \ell+2$ minor of $\tilde{A}^{(\ell)}_{1 \to 2}$ that contains the first row and treat it as a polynomial in $b_{21}$. The roots of this polynomial give the $\ell+1$ possible values for $b_{21}$.
\end{theorem}
\begin{proof} In the proof of Theorem \ref{thm:rank_condition} we show that the columns of $A^{(\ell)}_{1 \to 2}$ lie in the span of the columns of the matrix $M$ defined in \eqref{eq:matrix_M}.
More precisely, denoting by  $m_1, \dots, m_{\ell+1}$ the columns of $M$,  the $i$th column of $A^{(\ell)}_{1 \to 2}$ is 
\begin{equation*}
b_{21}^{i-1}m_1 + b_{23}^{i-1}m_2 + \cdots + b_{2,2+\ell}^{i-1}m_{\ell+1}.
\end{equation*}
Matching to this, $\left(\tilde{A}^{(\ell)}_{1 \to 2}\right)_{1, i} = b_{21}^{i-1} \cdot 1$ such that the columns of $\tilde{A}^{(\ell)}_{1 \to 2}$ are contained in 
\begin{equation*}
\text{span}\left( \left\{
\begin{pmatrix}
1\\
m_1
\end{pmatrix},
\begin{pmatrix}
0\\
m_2
\end{pmatrix}, 
\dots,
\begin{pmatrix}
0\\
m_{\ell+1}
\end{pmatrix}
\right\}\right).  
\end{equation*}
Consequently, the rank of $\tilde{A}^{(\ell)}_{1 \to 2}$ is at most $\ell+1$ and each minor of size $\ell+2$ vanishes. Similarly to the proof of Theorem \ref{thm:rank_condition}, we can show that generically, the minor is not the zero polynomial, which concludes the proof.
\end{proof}
For example, for $\ell=1$, $\tilde{A}^{(\ell)}_{1 \to 2}$ has size $4 \times 3$ and rank 2. The minors of size $3 \times 3$ provide the following three equations for $b_{21}$:
\begin{align*}
b_{21} ^2 \left(c_{1112} c_{112}-c_{1122} c_{111}\right)+b_{21}  \left(c_{1222} c_{111}-c_{1112} c_{122}\right)-c_{1222} c_{112}+c_{1122} c_{122} &=0,\\
b_{21} ^2 \left(c_{1111} c_{112}-c_{1112} c_{111}\right)+b_{21}  \left(c_{1122} c_{111}-c_{1111} c_{122}\right)-c_{1122} c_{112}+c_{1112} c_{122} &=0,\\
b_{21} ^2 \left(c_{1111} c_{1122}-c_{1112}^2\right)+b_{21}  \left(c_{1112} c_{1122}-c_{1111} c_{1222}\right)-c_{1122}^2+c_{1112} c_{1222} &= 0,
\end{align*}
These three equations are equivalent in the sense that their solutions coincide.
\subsection{Cumulants of the Latents}
Knowing the edge weights, we can determine certain cumulants of the latents and of $\epsilon_1$. 
\begin{lemma}\label{lem:lin_eq_omega} Under the model $\mathcal{M}_{2, \ell}$, 
\begin{equation}
\begin{pmatrix}
1 & 1 & \dots & 1 \\
b_{21} & b_{23} & \dots & b_{2,2+\ell} \\
\vdots & \vdots & \ddots & \vdots \\
b_{21}^{(k-1)} & b_{23}^{(k-1)} & \dots & b_{2,2+\ell}^{(k-1)}
\end{pmatrix}
\begin{pmatrix}
\omega_{1\dots 1} \\
\omega_{p+1,\dots,p+1}\\
\vdots \\
\omega_{p+\ell,\dots p+\ell}
\end{pmatrix}
=
\renewcommand{\arraystretch}{1.5} 
\begin{pmatrix}
c^{(k)}_{11 \dots 11} \\
c^{(k)}_{11\dots 12} \\
\vdots \\
c^{(k)}_{12\dots 22}
\end{pmatrix}.
\renewcommand{\arraystretch}{1} 
\end{equation}
This equation system is generically uniquely solvable if $k \geq \ell + 1$.
\end{lemma}
This result is a direct consequence of \eqref{lem:parametrization}.
\section{Arbitrary Number of Variables}
We aim to use the above results and an iterative procedure to estimate the causal order and all causal effects within an arbitrarily large graph. 
For now, we focus on finding only one valid choice for $B$, and we are indifferent to whether this choice corresponds to the sparsest possible graph. All other compatible options, particularly the sparsest ones, can be easily inferred from one choice, as laid out in Section \ref{sec:lin_sem}. 
The first step consists of determining a source $s$, the latents pointing to it, and all causal effects from the source and those latents on the remaining nodes.
\subsection{Inferring a Source and its Effects}
A crucial factor facilitating our strategy is that the marginal distribution of every pair of observed nodes $(v, w)$ again satisfies a linear structural equation model. Denote by $(Z_1, \dots, Z_{p+\ell}) = (X_1, \dots, X_p, L_1, \dots, L_{\ell})$  all observed and latent nodes and define the set of common confounders of $v$ and $w$ as
\begin{align*}
    \text{conf}(w,v) = \{Z_j \neq X_v, X_w: \ &\text{there exist two directed paths  $\pi_v, \pi_w$ from $Z_j$ to $v,w$,} \\&\text{respectively, not sharing any node apart from $Z_j$}\}.
\end{align*}
\begin{lemma}\label{lemma:marg_distr} Assume that $P^X$ follows a linear structural equation model consistent with a DAG $\mathcal{G}$ and that $X_v$ is a non-descendant of $X_w$. Then, the marginal distribution of $(X_v,X_w)$ lies in $\mathcal{M}_{2, |\text{conf}(v,w)|}$. If $v$ is a source, the parameters in the marginal model are given by
\begin{equation}\label{eq:parameters_marginal_model}
\begin{aligned}
   \ell' &= |\text{conf}(w,v)|, \\
  (L_1', \dots, L_\ell') &= (\eta_j: j \in\text{conf}(w,v)), \\
  \epsilon_1' &= \sum_{j \in \text{an}(v)\setminus\text{conf}(w,v)} \eta_j, \\
  \epsilon_2' &= \sum_{j \in \text{an}(w)\setminus\text{an}(v)}  \eta_j, \\
  b_{21}' &= b_{wv}, \text{ and} \\
  (b_{2,1+2}', \dots,b_{2,\ell+2}') &=(b_{wj}, j \in \text{conf}(w,v)).
\end{aligned}
\end{equation}
\end{lemma}
The proof and the parameters in the case that $v$ is no source can be found in the Appendix.

Using the lemma, for a pair of nodes $(v, w)$, we can identify which one is the ancestor by sequentially testing if $A^{(\ell)}_{v \to w}$ or $A^{(\ell)}_{w \to v}$ drops rank for $\ell=0,1,\dots$. If $A^{(\ell)}_{v \to w}$  drops rank for lower $\ell$, $v$ is the ancestor. 
In particular, a source can be found.
\begin{lemma} A node $s$ is a source if and only if for all other nodes $w \in [p]\setminus\{s\}$, $$\min\{\ell: \text{rank}(A^{(\ell)}_{s \to w}) = \ell+1\} \leq \min\{\ell: \text{rank}(A^{(\ell)}_{w \to s}) = \ell+1\}.$$
\end{lemma}

\begin{figure}
\centering
\begin{subfigure}[b]{0.4\textwidth}
  \centering
\begin{tikzpicture}[thick]
            \node[vertex] (1) at (0,0) {$s$}; 
            \node[vertex, fill=gray!20] (L1) at (1,1.2) {$L_1$};
            \node[vertex] (2) at (1.6,-0.5) {$w_1$};
            \node[vertex] (3) at (-1.6,-0.5) {$w_2$};
            \node[vertex, fill=gray!20] (L2) at (-1,1.2) {$L_2$};
            \draw[edge] (1) -- (2);
            \draw[edge] (L1) -- (2);
            \draw[edge] (L1) -- (1);
            \draw[edge] (L2) -- (1);
            \draw[edge] (L2) -- (3);
            \draw[edge] (1) -- (3);
        \end{tikzpicture}
        \subcaption{Pairwise confounding}\label{fig:Two_latents}
\end{subfigure}\hspace*{1cm}
\begin{subfigure}[b]{0.4\textwidth}\centering
  \begin{tikzpicture}[thick]
            \node[vertex] (1) at (0,0) {$s$}; 
            \node[vertex, fill=gray!20] (L1) at (3.1,1) {$L_1$};
            \node[vertex] (2) at (2,0) {$w_1$};
            \node[vertex] (3) at (2.5,-1.5) {$w_2$};
            \draw[edge] (1) -- (2);
            \draw[edge] (2) -- (3);
            \draw[edge] (L1) -- (2);
            \draw[edge] (L1) -- (1);
            \draw[edge] (L1) -- (3);
            \draw[edge] (1) -- (3);
        \end{tikzpicture}
        \subcaption{One latent variable }\label{fig:One_overall_latent}
\end{subfigure}
        \caption{Two graphs with the same number of confounders between the source and its descendants.}\label{fig:Same_pairwise_structure}
\end{figure}
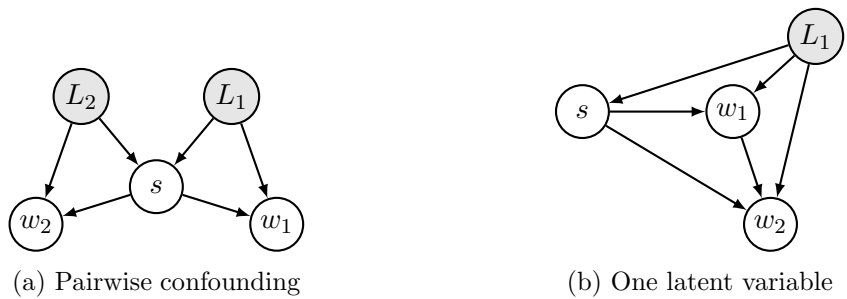
Applying this criterion to every pair of nodes, we simultaneously derive the number of confounders between the source and any other node $w$. Then, we can use the results from the previous section to estimate all parameters in the marginal model for $(w,s)$, that is, all edge weights $b_{wj}'$ for $j \in \text{conf}(w,s) \cup \{s\}$ and the cumulants of order at least $\ell+1$ of all $\eta_j(w,s) \in \text{conf}(w,s) \cup \{\epsilon_s'\}$.

But how can these pairwise pictures be combined to one overall graph? For example, if there is one latent confounding the variables $s$ and $w_1$, and one confounding $s$ and $w_2$,  the overall graph could be 
either of the graphs depicted in Figure \ref{fig:Same_pairwise_structure}.
To differentiate between these two models, we examine the cumulants in the marginal models: For each $w$ and each $\eta_j(w,s) \in \text{conf}(w,s)  \cup \{\epsilon_s'\}$, we collect its cumulants from order $\ell+1$ up to some fixed order $k_\text{max}$ into one cumulant vector $\omega_j(w,s)'$. If the right graph is correct, the same cumulant vector will be present in both marginal models, that is, $\omega_j(2,1)'= \omega_i(3,1)'$ for some choice of $i, j \in [2]$. In contrast, in the left graph, the cumulant vectors generically differ.
This generalizes to arbitrary graphs: Whenever a latent variable $L$ is an ancestor of $s$ and some other observed variables $w_1, \dots, w_m$, the same cumulant vector must occur in all the corresponding marginal models. Thus, by aligning the cumulants, we can associate the latent variables with their descendants. Note that for each marginal model, one of the estimated cumulant vectors does not correspond to a latent variable but to $\epsilon_s'$. This cumulant vector can differ for different $w$ even if $\epsilon_s$ is an ancestor of all of them since the noise terms $\epsilon'_s(w, s) = \sum_{j \in \text{an}(v)\setminus\text{conf}(s, w)} b_{sj} \eta_j$ in the marginal models might differ.

Because the causal effects in the marginal models coincide with those in the overarching model, we can now fill in all columns of $B$ corresponding to the source and its latent parents. Given our focus on a single valid option for $B$, we enumerate the latents $L_1, \dots L_m$ arbitrarily. This enumeration fixes the arrangement of the corresponding columns in $B$ since $\eta$ and $B$ can only be permuted simultaneously.
\subsection{Next Iteration}
In order to proceed to the next iteration, we want to remove the source and its parents from the data and compute $X_w^{(1)} = X_w -  b_{ws}\epsilon_s - \sum_{j=p+1}^m b_{wl}L_l$ for $w \neq s$ because the joint distribution of those random variables satisfies a structural equation model for the graph with the source and its parents removed. 
\begin{lemma}\label{lem:cumulant_regression} Assume there are $m$ confounders $L_1, \dots, L_m$  pointing to the source $s$, and let $b_{wj}, j={s, p+1, \dots, m}$, be a valid choice of parameters found in the first iteration. Now consider the joint distribution of all
\begin{align}\label{eq:cumulant_regression}
  X_w^{(1)} = X_w -  b_{ws}\epsilon_s - \sum_{j=p+1}^m b_{wj}L_j
  \end{align}
for $w \neq s$. Then, this distribution satisfies the structural equation model belonging to $\mathcal{G}$ with the nodes $s$, $L_1, \dots, L_m$ and all adjacent edges erased.
\end{lemma} 
\begin{proof} Obtain $B', \eta', \Lambda', \Gamma'$ from $B, \eta, \Lambda, \Gamma$ by removing all the rows and columns that correspond to $s$ and its latent parents, that is, remove row $1$ and columns $1, p+1, \dots, p+m$ from $B$, the entries $\eta_1, \eta_{p+1}, \dots, \eta_{p+1}$ from $\eta$, row and column $1$ from $\Lambda$, and columns $1, \dots, m$ from $\Gamma$. 
Then, by definition of $X^{(1)}$
  \begin{align*}
    X^{(1)} = B' \eta'.
  \end{align*}
Moreover, $\Lambda$ is lower triangular, so $\Lambda'$ is still invertible and $(I-\Lambda')^{-1} = (I-\Lambda)^{-1}_{1:,1:}$. Therefore, 
$$B' = (I-\Lambda')^{-1} (I, \Gamma').$$ The sparsity pattern of $\Lambda', \Gamma'$ corresponds to the graph $\mathcal{G}$ with the nodes $s$, $L_1, \dots, L_m$ and all adjacent edges removed, which concludes the proof.
\end{proof}
Given that we only have access to the observed components $X$, acquiring data sampled according to the distribution of $X^{(1)}$ is not feasible. However, to apply our procedure, it suffices to know the cumulants. All causal effects appearing in the formula \eqref{eq:cumulant_regression} for the distribution can be inferred and computing the cumulant commutes with summation if the summands are independent random variables. Therefore, the only remaining question is whether the cumulants of the exogenous sources can be estimated.
\begin{lemma}\label{lem:regression_possible} Denote by $L_1, \dots ,L_m$ the latents parents of the source $s$. If there exist $m$ distinct observed nodes
$v_1, \dots, v_m \in [p]\setminus\{s\}$ such that $v_i$ is a child of $L_i$, then all cumulants of $\epsilon_s, L_1, \dots, L_m$ of order two and higher can be estimated.
\end{lemma}
\begin{proof} 
 We first consider the second-order cumulants. Without loss of generality, let $1$ be the source. From Lemma \ref{lem:parametrization},
\begin{align}\label{eq:omegas_overall_model} 
  \begin{pmatrix}
   1 & 1 & \dots & 1 \\
   b_{21} & b_{2,p+1} & \dots & b_{2,p+m} \\
   \vdots & \vdots & \ddots & \vdots \\
   b_{p1} & b_{2,p+1} & \dots & b_{p,p+m} 
  \end{pmatrix}
  \begin{pmatrix}
    \omega_{11} \\
    \omega_{p+1,p+1} \\
    \vdots \\
    \omega_{p+m,p+m}
  \end{pmatrix}
  =
  \begin{pmatrix}
    c_{11} \\
    c_{12} \\
    \vdots \\
    c_{1p}
  \end{pmatrix}
 \end{align} 
Denote the matrix in the equation by $\tilde{B}$. Its transpose  coincides with the columns $1, p+1, \dots, p+m$ of $B=(I-\Lambda)^{-1}(I, \Gamma)$. Since the first factor $(I-\Lambda)^{-1}$ is invertible, the rank of $\tilde{B}$ coincides with the rank of $M = (I, \Gamma)_{:,1, p+1, \dots, p+m}$.  Under the assumption in the Lemma, the columns and rows of $M$ can be permuted such that its diagonal is non-zero. Combining this with the genericity assumption, it follows that $M$ has full rank. Hence, $\tilde{B}$ has rank $\min(p, 1+m)=1+m$, which is the number of its columns.

For higher order cumulants, the same argument applies by considering the equations defining $c^{(k)}_{1\dots 1i}$ instead of $c^{(2)}_{1i}, i \neq 1$.
\end{proof}
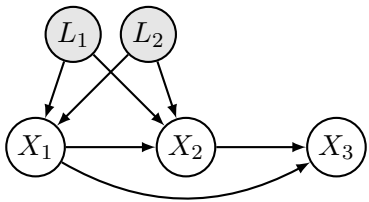
\begin{figure}
    \centering
    \begin{tikzpicture}[thick]
      \node[vertex] (1) at (-2,0) {$X_1$};
      \node[vertex] (2) at (0,0) {$X_2$};
      \node[vertex] (3) at (2,0) {$X_3$};
      \node[vertex, fill=gray!20] (4) at (-1.5, 1.5) {$L_1$};
      \node[vertex, fill=gray!20] (5) at (-0.5, 1.5) {$L_2$};
      \draw[edge] (1) --  (2);
      \draw[edge] (1) to[bend right=30]  (3);
      \draw[edge] (2) --  (3);
      \draw[edge] (4) -- (1);
      \draw[edge] (4) -- (2);
      \draw[edge] (5) -- (1);
      \draw[edge] (5) -- (2);
    \end{tikzpicture}
    \caption{Cumulant estimation might fail}
    \label{fig:regression_not_possible}
\end{figure}
\begin{example}
    As an illustrating example of the necessity of the assumption in the lemma, consider the graph in Figure \ref{fig:regression_not_possible}. Here, the linear equation system for the second-order cumulants of $\text{exog}(1)$ reads 
\begin{align*}
    \begin{pmatrix}
      c_{11} \\
      c_{12} \\
      c_{13} \\
    \end{pmatrix} =
  \begin{pmatrix}
    1 & 1 & 1 \\
    b_{21} & b_{24} & b_{25} \\
    b_{31} & b_{34} & b_{35} \\
    \end{pmatrix}
    \begin{pmatrix}
      \omega_{11} \\
      \omega_{44} \\
      \omega_{55} \\
    \end{pmatrix} 
    = 
  \begin{pmatrix}
    1 & 1 & 1 \\
    b_{21} & b_{24} & b_{25} \\
    \lambda_{32}b_{21} + \lambda_{31} & \lambda_{32}b_{24} + \lambda_{31} & \lambda_{32}b_{25} + \lambda_{31} \\
    \end{pmatrix}
    \begin{pmatrix}
      \omega_{11} \\
      \omega_{44} \\
      \omega_{55} \\
    \end{pmatrix},
  \end{align*} where the last equality holds 
since none of the latents points to node $3$. Consequently, the last row is a linear combination of the first two rows.
If one of the two latents also points to $3$ the equation system would become invertible.
\end{example}

\section{Practical Implementation}
Putting together the previous sections' results essentially leads to our proposed algorithm, as outlined in Algorithm \ref{alg:ReLVLiNGAM}. \footnote{Our code is available at \url{https://github.com/DanielaSchkoda/ReLVLiNGAM}.}
\begin{algorithm}[t]
  \caption{ReLVLiNGAM}\label{alg:ReLVLiNGAM}
\begin{algorithmic}[1]
\State \textbf{input} Data $X \in \mathbb{R}^{n \times p}$,  bound on pairwise confounding $\ell_{\max}$.
  \State $R \gets \{1, \dots, p\}$.
  \State $\hat{C}^{(2)}, \dots, \hat{C}^{(k_2)} \gets $ sample cumulants of $X$.
  \Repeat
  \State $s \gets$ find source($X$, $\ell_{\max}$).
  \State Estimate $\hat{b}_{ws}, \hat{b}_{w,1+p}, \dots, \hat{b}_{w,1+p+\ell_{sw}}$ for $w \in R\setminus \{s\}$. \hfill (Theorem \ref{thm:effect_sizes})
  \State Estimate the cumulants of the exogeneous sources in all marginal models. \hfill (Lemma \ref{lem:lin_eq_omega})
  \State Align the latent variables and fill in $\hat{B}_{:,(s,p+1, \dots, p+m_s)}$.
  \State Estimate the cumulants of $\text{exog}(s)\cup \{\epsilon_s\}$ in the overall model. \hfill (Equation \ref{eq:omegas_overall_model})
  \State $\hat{C}^{(2)}, \dots, \hat{C}^{(k_2)} \gets$ estimated cumulants of $X - \hat{B}_{:,(s,p+1, \dots, p+m_s)}X$. \hfill (Lemma \ref{lem:cumulant_regression})
  \State $R \gets R\setminus \{s\}$.
  \Until{$|R| = 1$.}
  \State From $\hat{B}$, calculate all possible solutions $\hat{B}^{(1)}, \dots, \hat{B}^{(h)}$.  \hfill (Section \ref{sec:lin_sem})
  \State \Return Estimated path matrices $\hat{B}^{(1)}, \dots, \hat{B}^{(h)}$.
\end{algorithmic}
\end{algorithm}
\begin{algorithm}[t]
  \caption{Find source}
\begin{algorithmic}[1]
  \State \textbf{input} Data $X \in \mathbb{R}^{n \times |R|}$,  bound on pairwise confounding $\ell_{\max}$.
  \For{each pair $(v, w)$}
  \State $\ell_{vw} \gets \min(\{l = 0, \ldots, \ell_{\max}: \text{rank}(A^{(\ell)}_{v \to w}) \leq l+1$\})
  \EndFor
  \State  \Return $v$ with such that $\sum_{w \neq v}\ell_{vw}$ is minimal.
\end{algorithmic}
\end{algorithm}
However, transitioning from theoretical results to finite sample size, some practical questions arise. Let $n$ be the sample size and denote the observed data matrix by $X \in \mathbb{R}^{p\times n}$. The first step of the algorithm is estimating the cumulants, which we achieve using the plug-in statistic, where sample moments of $X$ are calculated and then plugged into the equations for the cumulants. Finding the source relies on the rank condition from Theorem \ref{thm:rank_condition}. We compute the singular values $\sigma_1, \dots \sigma_{\ell+2}$ of $A^{(\ell)}_{v \to w}$ and accept the hypothesis that $\text{rank}(A^{(\ell)}_{v \to w}) \leq \ell+1$ if $\sigma_{\ell+2}/\sigma_1$ falls below or is equal to  a threshold $T$. Initially, we set $T = {0.08}{n^{-0.2}}$ and adjust it to $T = {0.2}(i-1){n^{-0.2}}$ in each later iteration $i$ to account for the expected increase in error. To ensure scale-freeness in this rank test, when forming $A^{(\ell)}_{v \to w}$, we do not use the cumulants of $X$ but of its scaled version $\tilde{X} = (X_1/\hat\sigma_1, \dots, X_p/\hat\sigma_p)$, where $\hat\sigma_i$ is the empirical variance of $X_i$. When faced with a non-unique minimum in Line 5, amongst all minima $v$, we opt for the one with the lowest average of ratios $\sum_{w \neq v}\sigma_{\ell+2}(v, w)/\sigma_1(v, w)$.

To find the causal effects $\hat{b}_{ws}, \hat{b}_{w,1+p}, \dots, \hat{b}_{w,1+p+\ell_{sw}}$, Theorem \ref{thm:effect_sizes} provides several equivalent polynomial equations, whose coefficients are cumulants. We confine to the equations that feature the most lower-order cumulants and take the mean of the solutions across the equations. For example, for $\ell=1$, the equations are all $3 \times 3$ minors of the matrix
\begin{align*}
  \begin{pmatrix}
    1 & b_{21} & b_{21}^2 \\
    c_{111} & c_{112} & c_{122} \\
    c_{1111} & c_{1112} & c_{1122} \\
    c_{2111} & c_{2112} & c_{2122}\\
  \end{pmatrix}
\end{align*}
that include the first row. We only use the minor selecting the rows $1, 2, 3$ and the minor selecting the rows $1, 2, 4$.

The next step groups the latents. We first estimate the cumulants of the exogenous sources using Lemma \ref{lem:lin_eq_omega} and then align the latents by aligning the cumulants. Two cumulant vectors are considered to match if their Euclidean distance falls below $0.1$.  This threshold could be further tuned but for our setting of standardized observations, we obtain good experimental performance from the given choice.

The increase in error in each iteration motivates a final minor adjustment in the algorithm: In iteration $i$, we already compute $\hat\ell_{vw}{(i)}$ for all $v, w \in R$. So, we can reuse this information when estimating $\ell_{vw}$ again in the next iteration: In iteration $i+1$, we set $$\ell_{\max} = \hat\ell_{vw}{(i)} - |\{L: L \text{ common confounder of $v$ and $w$ found in iteration $i$}\}|.$$

Combining all the results from above proves that the algorithm returns the true path matrices for infinite sample size.
\begin{theorem} Given the exact cumulants of $P^X \in \mathcal{M}(G)$, and setting all thresholds in the algorithm to $0$, if the condition of Lemma \ref{lem:regression_possible} is satisfied in every iteration, Algorithm \ref{alg:ReLVLiNGAM} returns all path matrices compatible with $P^X$. Moreover, the algorithm recognizes the non-fulfillment of the condition, since, in this case, the linear equation system to estimate the cumulants is underdetermined.
\end{theorem}
\section{Simulations}
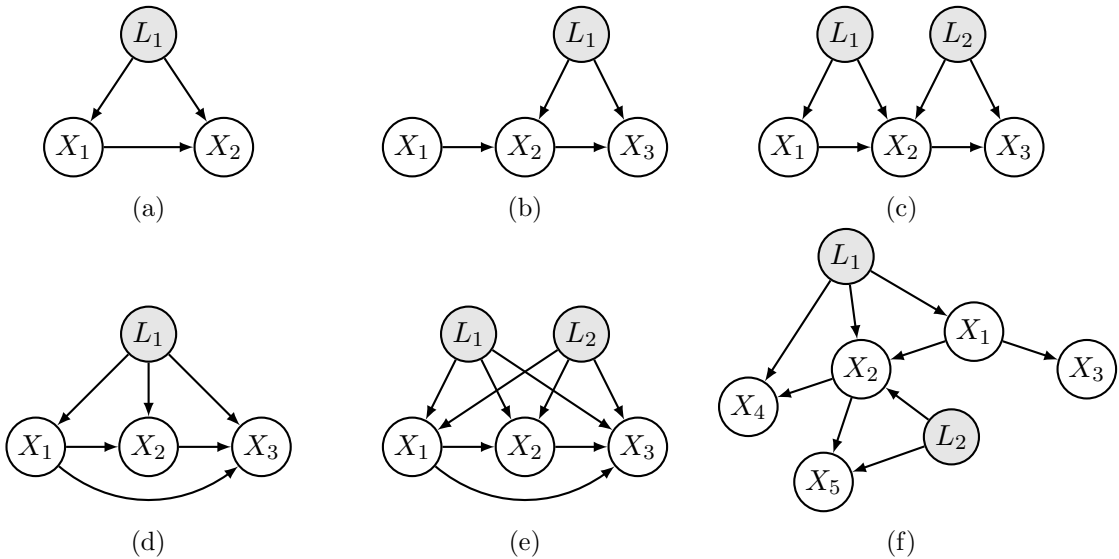
\begin{figure}[t]
  \centering
  \begin{subfigure}[t]{0.32\textwidth}
    \centering
    \begin{tikzpicture}[thick]
      \node[vertex] (1) at (-1,0) {$X_1$};
      \node[vertex] (2) at (1,0) {$X_2$};
      \node[vertex, fill=gray!20] (3) at (0,1.5) {$L_1$};
      \draw[edge] (1) --  (2);
      \draw[edge] (3) -- (1);
      \draw[edge] (3) -- (2); 
    \end{tikzpicture}
    \caption{}
  \end{subfigure}
  \begin{subfigure}[t]{0.32\textwidth}
    \centering
    \begin{tikzpicture}[thick]
      \node[vertex] (1) at (-1.5,0) {$X_1$};
      \node[vertex] (3) at (1.5,0) {$X_3$};
      \node[vertex] (2) at (0,0) {$X_2$};
      \node[vertex, fill=gray!20] (4) at (0.75,1.5) {$L_1$};
      \draw[edge] (1) --  (2);
      \draw[edge] (2) -- (3);
      \draw[edge] (4) -- (2);
      \draw[edge] (4) -- (3);
    \end{tikzpicture}
   \caption{ }
  \end{subfigure}
  \begin{subfigure}[t]{0.32\textwidth}
    \centering
    \begin{tikzpicture}[thick]
      \node[vertex] (1) at (-1.5,0) {$X_1$};
      \node[vertex] (3) at (1.5,0) {$X_3$};
      \node[vertex] (2) at (0,0) {$X_2$};
      \node[vertex, fill=gray!20] (4) at (-0.75,1.5) {$L_1$};
      \node[vertex, fill=gray!20] (5) at (0.75,1.5) {$L_2$};
      \draw[edge] (1) --  (2);
      \draw[edge] (2) --  (3);
      \draw[edge] (4) -- (1);
      \draw[edge] (4) -- (2);
      \draw[edge] (5) -- (2);
      \draw[edge] (5) -- (3);
    \end{tikzpicture}
    \caption{}
  \end{subfigure}
  \begin{subfigure}[t]{0.32\textwidth}
    \centering
    \begin{tikzpicture}[thick]
      \node[vertex] (1) at (-1.5,0) {$X_1$};
      \node[vertex] (3) at (1.5,0) {$X_3$};
      \node[vertex] (2) at (0,0) {$X_2$};
      \node[vertex, fill=gray!20] (4) at (0,1.5) {$L_1$};
      \draw[edge] (1) --  (2);
      \draw[edge] (1) to[bend right=40]  (3);
      \draw[edge] (2) --  (3);
      \draw[edge] (4) -- (1);
      \draw[edge] (4) -- (2);
      \draw[edge] (4) -- (3);
    \end{tikzpicture}
    \caption{}
  \end{subfigure}
  \begin{subfigure}[t]{0.32\textwidth}
    \centering
    \begin{tikzpicture}[thick]
      \node[vertex] (1) at (-1.5,0) {$X_1$};
      \node[vertex] (3) at (1.5,0) {$X_3$};
      \node[vertex] (2) at (0,0) {$X_2$};
      \node[vertex, fill=gray!20] (4) at (0.75,1.5) {$L_2$};
      \node[vertex, fill=gray!20] (5) at (-0.75,1.5) {$L_1$};
      \draw[edge] (1) --  (2);
      \draw[edge] (2) --  (3);
      \draw[edge] (4) -- (1);
      \draw[edge] (4) -- (2);
      \draw[edge] (4) -- (3);
      \draw[edge] (5) -- (1);
      \draw[edge] (5) -- (2);
      \draw[edge] (5) -- (3);
      \draw[edge] (1) to[bend right=40]  (3);
    \end{tikzpicture}
    \caption{}
  \end{subfigure}
  \begin{subfigure}[t]{0.32\textwidth}
    \centering
    \begin{tikzpicture}[thick]
      \node[vertex] (1) at (0,0.5) {$X_1$};
      \node[vertex] (2) at (-1.5,0) {$X_2$};
      \node[vertex] (3) at (1.5,0) {$X_3$};
      \node[vertex] (4) at (-3,-0.5) {$X_4$};
      \node[vertex] (5) at (-2,-1.5) {$X_5$};
      \node[vertex, fill=gray!20] (L1) at (-1.7,1.5) {$L_1$};
      \node[vertex, fill=gray!20] (L2) at (-0.3,-0.9) {$L_2$};
      \draw[edge] (1) --  (2);
      \draw[edge] (1) -- (3);
      \draw[edge] (2) --  (5);
      \draw[edge] (2) --  (4);
      \draw[edge] (L1) -- (1);
      \draw[edge] (L1) -- (2);
      \draw[edge] (L1) -- (4);
      \draw[edge] (L2) -- (2);
      \draw[edge] (L2) -- (5);
    \end{tikzpicture}
    \caption{}
  \end{subfigure}
  \caption{Settings.}
  \label{fig:graphs_simulation_study}
\end{figure}
In simulation studies, we compare the performance of our method with the RICA method proposed in \cite{Salehkaleybar2020}.  Specifically, we use RICA as an idealized benchmark by providing it with the true number of latent variables $\ell$.
Before delving into the details of the simulation setup, we highlight some relevant aspects of RICA: Given the number of latents, it leverages overcomplete independent component analysis, aiming to compute a path matrix $\hat{B}$ such that the corresponding exogenous sources $\hat{\eta}$ are close to having independent components and far from being Gaussian. 
Since the resulting $\hat{B}$ does not obey any sparsity constraints, as a final step, bootstrap and a t-test are employed to prune non-significant causal effects. Nonetheless, the resulting $\hat{B}$ might correspond to a cyclic graph. 
\begin{figure}
\begin{subfigure}{0.32\textwidth}
      \includegraphics[width=\linewidth]{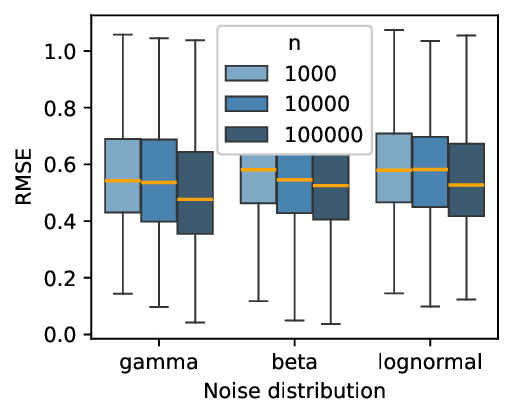}
      \caption{RMSE}
    \end{subfigure}
    \begin{subfigure}{0.32\textwidth}
      \includegraphics[width=\linewidth]{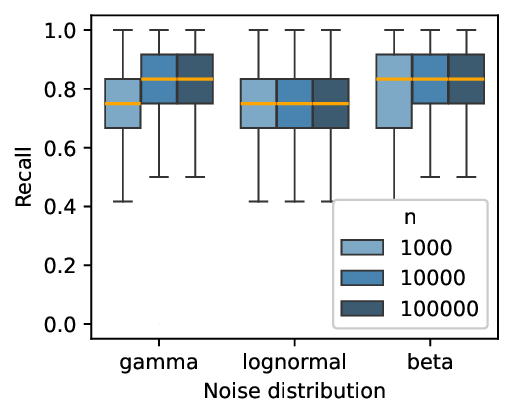}
      \caption{Recall}
    \end{subfigure}
    \begin{subfigure}{0.32\textwidth}
      \includegraphics[width=\linewidth]{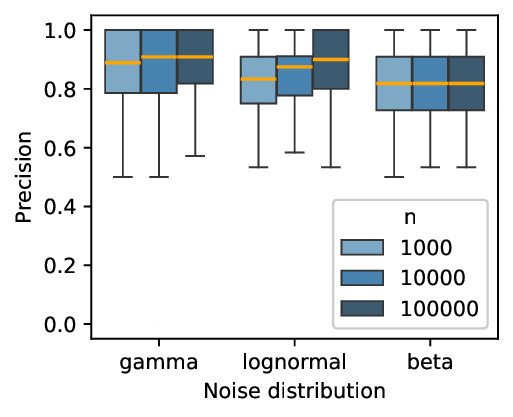}
      \caption{Precision}
    \end{subfigure}
    \caption{Simulation results for setting e) with ReLVLiNGAM. The distribution of $\eta$ varies, and $\ell_{\text{max}}$ is correctly specified.}
    \label{fig:simulation_results_different_noise_distributions}
     \end{figure}

To sample the data, we use the six graphs shown in Figure \ref{fig:graphs_simulation_study}, the edge weights are chosen uniformly from $[-0.9, -0.5] \cup [0.5, 0.9]$, and $\eta$ is drawn from a gamma, log-normal, or beta distribution. Afterwards we randomly permute the variables $X_1, \dots, X_p$ to establish a random topological order. However, varying the noise distribution seems to have little impact on the results, as shown in Figure \ref{fig:simulation_results_different_noise_distributions}. Therefore, for all of the following simulation results, we focus on gamma-distributed $\eta$. For each setting, we perform $1000$ replications and measure precision and recall regarding the existence of causal paths and the RMSE of $\hat{B}$. Here, a subtlety arises from the non-uniqueness of $B$. To overcome the arbitrary rescaling, we have so far used the convention of setting the edge from a latent to its oldest child to $1$. However, since \cite{Salehkaleybar2020}'s method does not necessarily return an acyclic graph, this convention no longer makes sense. Instead, we follow their suggestion to divide each column of $B$ and $\hat{B}$ by the entry with the maximum absolute value. What remains is the possibility to permute the columns. For our method, we compute all $n_\mathcal{G}$ options for $\hat{B}_i$ as explained in section \ref{sec:lin_sem} and take the one that yields the smallest RMSE. Again, this procedure is not well defined for a potentially cyclic graph. Thus, for the RICA method, we consider any permutation of the columns. If $B$ and $\hat{B}$ differ in the number of columns, we pad the smaller matrix with zeros to compute the RMSE.

For RICA, we need to specify the overall number of latents $\ell$, while our algorithm requires an upper bound on the pairwise confounding $\ell_\text{max}$. We consider two options, namely the actual highest number of pairwise confounding within the graph, so $\ell_\text{max}=2$ in setting e) and $1$ in all remaining settings, as well as this actual value increased by one. 
\ref{fig:simulation_results_all_methods}
  \begin{figure}
    \centering
    \begin{subfigure}[t]{1\textwidth}
      \includegraphics[width=0.32\linewidth]{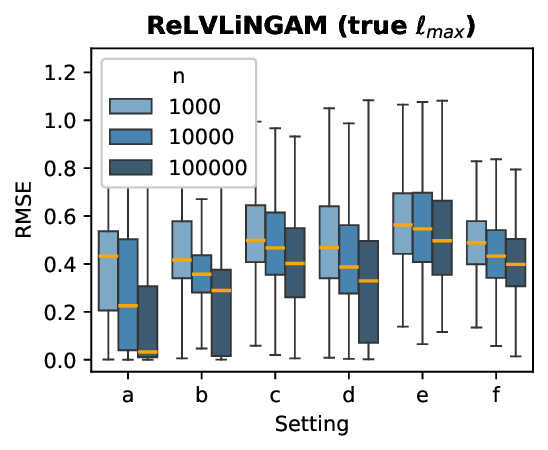}
      \includegraphics[width=0.32\linewidth]{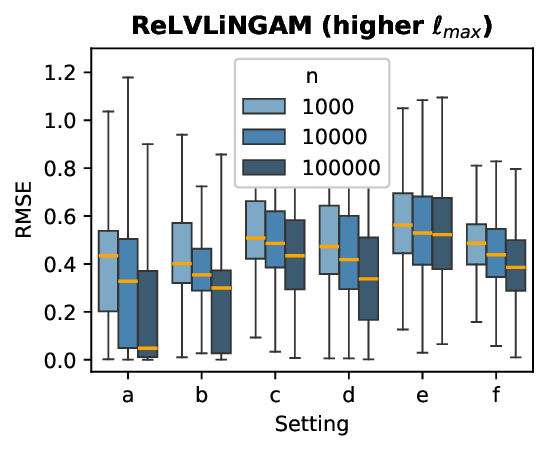}
      \includegraphics[width=0.32\linewidth]{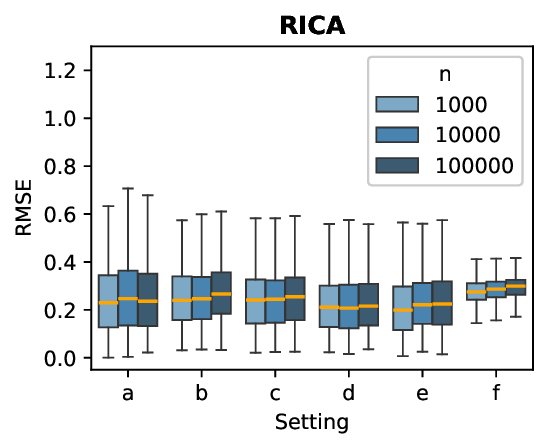}
      \caption{RMSE}
    \end{subfigure}
    
    \vspace{0.2cm}
    
    \begin{subfigure}[t]{1\textwidth}
      \includegraphics[width=0.32\linewidth]{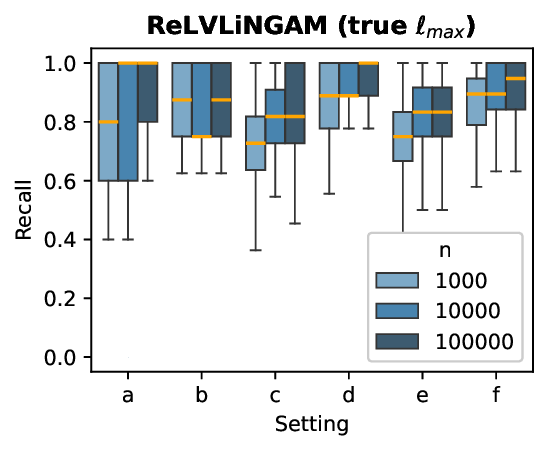}
      \includegraphics[width=0.32\linewidth]{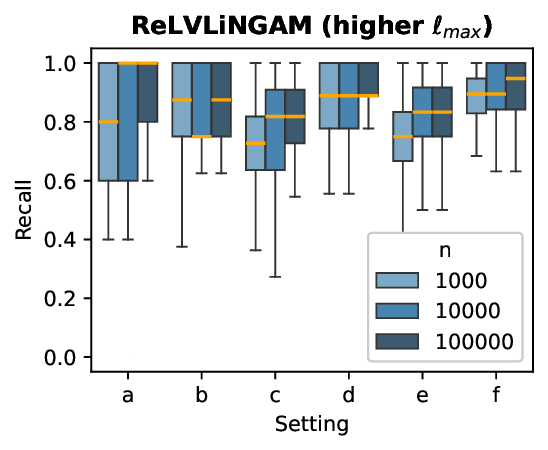}
      \includegraphics[width=0.32\linewidth]{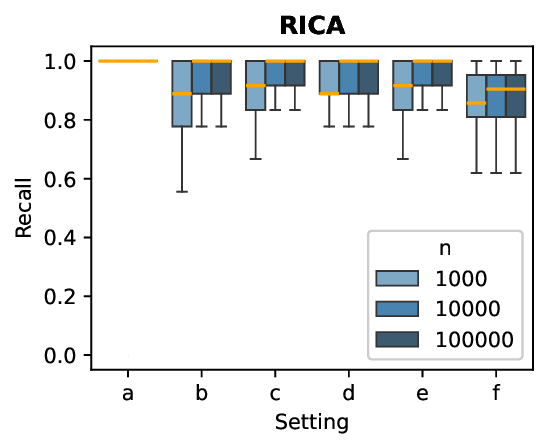}
      \caption{Recall}
    \end{subfigure}
    
    \vspace{0.2cm}
    
    \begin{subfigure}[t]{1\textwidth}
      \includegraphics[width=0.32\linewidth]{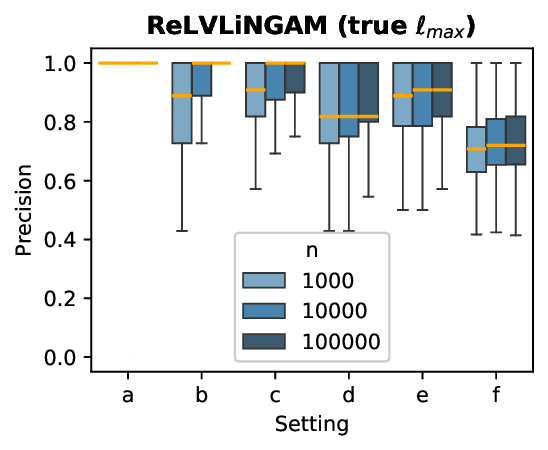}
      \includegraphics[width=0.32\linewidth]{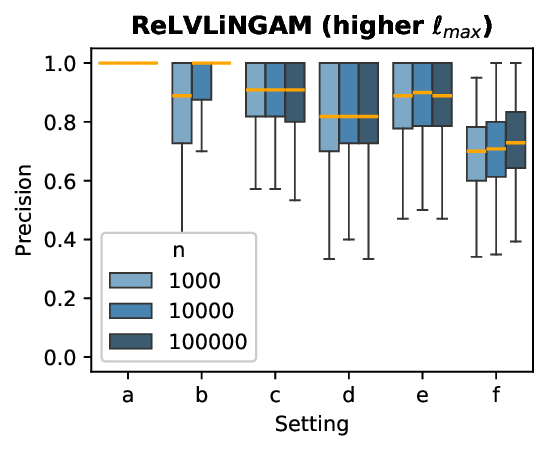}
      \includegraphics[width=0.32\linewidth]{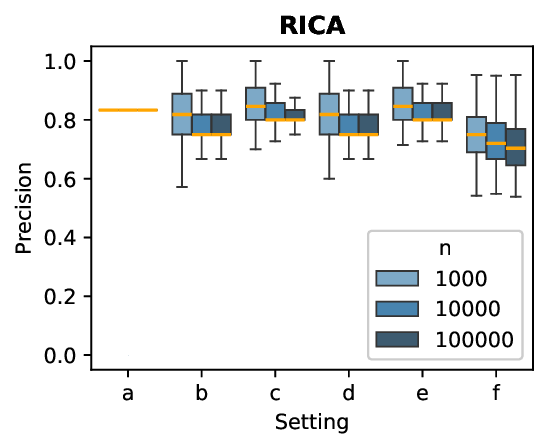}
      \caption{Precision}
    \end{subfigure}
    \caption{Simulation results for $\eta$ gamma-distributed and varying sample size.}
    \label{fig:simulation_results_all_methods}
     \end{figure}
     
RICA tends to excel in RMSE, particularly for low sample sizes. However, its performance does not improve notably with higher sample sizes. This difference in performance, especially the higher variance for our method, might be attributed to the algorithm architectures: Our approach first estimates the graph's structure and only afterward infers the edge weights, where substantial errors can be expected whenever there are errors in the graph estimation. In contrast, RICA searches directly for the best-fitting path matrix, knowing the correct $\ell$. Turning to precision and recall, ReLVLiNGAM exhibits higher precision, while RICA outperforms in terms of recall. This outcome is not surprising given that our algorithm is forced to produce a DAG, thereby constraining nearly half of the entries of $B$ to be $0$.

Across the different settings, RICA shows similar performance, while our method declines in performance for larger numbers of nodes as errors accumulate throughout the iterations. Nevertheless, causal effects from nodes positioned early in the topological order may still be estimated reasonably accurately, whereas, with incorrectly specified ${\ell}$ in RICA, we can expect the entire estimated path matrix to deviate significantly from the truth.

Comparing the two choices of $\ell_{\text{max}}$, choosing a higher value only marginally reduces performance, indicating that our method is robust to misspecified $\ell_\text{max}$.  In other words, our ReLVLiNGAM achieves state-of-the-art estimation accuracy without needing to know (or very accurately estimating) the number of latent variables.

\section{Discussion}

We demonstrated that in a linear non-Gaussian structural equation model featuring latent confounding, the graph structure can be uniquely identified based on cumulants of the observed distribution.  In doing so, we investigated which order of cumulants is sufficient for this purpose and showed how this order is determined by the number of latent variables. 

For causal discovery, we introduced a consistent algorithm that iteratively identifies a source node of a causal diagram and infers the number of its latent parents using a rank constraint on a matrix formed from cumulants. For estimation of the source's causal effects on its descendants, the algorithm leverages suitable polynomial equations. In our simulations, we demonstrated that our algorithm accurately identifies the number of latent variables, even when the upper bound on pairwise confounding is not tightly set, which represents a significant advantage over existing OICA approaches. In addition, our method improves on that proposed by \cite{Cai2023}, by relaxing the assumptions the true graph has to satisfy. Specifically, our only requirement is that locally the number of latent variables is lower than the number of observed variables.

We remark that the identifying equations we derived can also be used to estimate specific causal effects when the graph is already known. More generally, the iterative nature of our method could be exploited to incorporate prior knowledge.  Finally, we highlight that an interesting problem for future research would be to develop extensions of our algorithm that are able to accommodate sparse, high-dimensional settings.


\acks{This project has received funding from the European Research Council (ERC) under the European Union’s Horizon 2020 research and innovation programme (grant agreement No 883818). Daniela Schkoda acknowledges support by the DAAD programme Konrad Zuse Schools of Excellence in Artificial Intelligence, sponsored by the Federal Ministry of Education and Research. Elina Robeva was supported by an NSERC Discovery Grant (DGECR-2020-00338) and a Canada CIFAR AI Chair Award (AWD-028752 CIFAR 2024).}


\newpage

\appendix
\section{} This appendix contains the proofs of Lemma \ref{lem:number_of_sparsest_choices}, Theorem \ref{thm:rank_condition} and Lemma \ref{lemma:marg_distr}.\vskip 0.1in

\begin{proof}\textbf{of Lemma \ref{lem:number_of_sparsest_choices}} Assume $P^X \in \mathcal{M}_\mathcal{G}$ with path matrix $B$. Then, every other compatible path matrix $B'$ can be obtained by swapping the $v$th and $(w+p)$th columns in $B$ for some $v \in [p]$ and $L_w \in \text{exog}(v)$. In formulas $B=B' S$, where $S\in\mathbb{R}^{p+\ell\times p+\ell}$ is the permutation matrix obtained from the identity matrix by swapping columns $v$ and $w+p$.  We claim that the coefficients for the observed variables corresponding to $B'$ form the matrix  $\Lambda'=\Lambda + (\Gamma_{:,w}-e_v)\cdot(I-\Lambda)_{v,:}$. Indeed, \ \begin{equation}
\begin{aligned}\label{eq:other_compatible_Lambda}
  (I-\Lambda')B' &= (I-\Lambda - (\Gamma_{:,w}-e_v)\cdot(I-\Lambda)_{v,:})BS\\
  &= (I-\Lambda)BS - (\Gamma_{:,w}-e_v)\cdot(I-\Lambda)_{v,:}BS\\
  &= (I, \Gamma)S - (\Gamma_{:,w}-e_v)\cdot (I, \Gamma)_{v,:}S.
\end{aligned}
  \end{equation}
Note that for every $L_w \in \text{exog}(v)$, $v$ is its unique oldest child such that $\gamma_{vw}=1$. Hence $I_{vv}$ and $\gamma_{vw}$ coincide, yielding that $((I, \Gamma)_{v,:}S)=(I, \Gamma)_{v,:}$.
In addition, \begin{align}\label{eq:other_compatible_Lambda_2}(\Gamma_{:,w}-e_v)\cdot (I, \Gamma)_{v,:}
  &= (\Gamma_{:,w}-e_v)\cdot 
\begin{pmatrix}
  e_v^T & \gamma_{v,p+1} & \cdots & \gamma_{v,p+\ell}
\end{pmatrix}.
\end{align} 
Hence, 
\begin{equation*}
    [(I-\Lambda')B']_{:,:p} = [(I, \Lambda)S]_{:,:p} - (\Gamma_{:,w}-e_v)\cdot e_v^T = I.
\end{equation*}
which shows the claim. Knowing $\Lambda'$, we can calculate $\Gamma'$ as
\begin{align*}
    \Gamma' = (I - \Lambda') B'_{:,(p+1):}.
\end{align*}
Again using \eqref{eq:other_compatible_Lambda} and \eqref{eq:other_compatible_Lambda_2}, we see that 
\begin{align*}
    (I - \Lambda') B'_{:,(p+1):} = [(I, \Gamma)S]_{:, (p+1):} - (\Gamma_{:,w}-e_v)\cdot 
\begin{pmatrix}
\gamma_{v,p+1} & \cdots & \gamma_{v,p+\ell}
\end{pmatrix}.
\end{align*}
So, $(I-\Lambda')B' = (I, \Gamma')$ for $\Gamma'$ defined by 
\begin{align*}\Gamma'_{:,j}=\begin{cases}\Gamma_{:,j} -\gamma_{v,j}(\Gamma_{:,w}-e_v) &\text{for }j \neq w,\\ 
e_v-1(\Gamma_{:,j}-e_v) 
&\text{otherwise. }
\end{cases} 
\end{align*} 
To compare the sparsity patterns of $(\Lambda, \Gamma)$ and $(\Lambda', \Gamma')$, we write out the single entries:
\begin{align*}
  \lambda'_{ij}&=\lambda_{ij}+(\gamma_{iw}-\delta_{iv})(\delta_{vj}-\lambda_{vj}), \\
  \gamma'_{ij}&=\begin{cases}\gamma_{ij}-\gamma_{vj}(\gamma_{iw}-\delta_{iv})
  &\text{for }j \neq w,\\ 
    -\gamma_{ij}+2\delta_{iv}
    &\text{otherwise. }
  \end{cases} 
  \end{align*} 
From the genericity assumption, the graph $\mathcal{G}'$ encoding the sparsity pattern of $\Lambda'$ and $\Gamma'$ cannot contain fewer edges than $\mathcal{G}$. In particular, $\mathcal{G}$ is the unique minimal graph. Conversely, $\mathcal{G}'$ might contain additional edges. First, note that the part of the formula stating $\gamma_{ij} = -\gamma_{ij}+2\delta_{iv}$ for $j=w$ does not result in new edges since $\gamma_{vw}$ is already non-zero. Since $(\gamma_{iw}-\delta_{iv})(\delta_{vj}-\lambda_{vj}) \neq 0$ if and only if
$$i \neq v, \gamma_{iw}\neq 0, \text{ and either }v=j \text{ or } \lambda_{vj}\neq 0,$$ and $\gamma_{vj}(\gamma_{iw}-\delta_{iv}) \neq 0, j\neq w$ if and only if
$$\gamma_{vj} \neq 0, \gamma_{iw} \neq 0, i \neq v, \text{ and } j \neq w,$$ the graph $\mathcal{G}'$ incorporates
\begin{itemize}
    \item the edge $v \to i$ for $i\in [p]\setminus\{v\}$ whenever in the original graph $w \to i$;
\item the edge $j \to i$ for $i\in [p]\setminus\{v\}$, $j \in [p+\ell]\setminus\{w\}$ whenever in the original graph $j \to v$ and $w \to i$.
\end{itemize}
 Therefore, to not introduce new edges, we must swap $\epsilon_v$ with a latent $L_w$ whose children are already children of $v$. The other way round, all siblings of $v$ need to be children of $L_w$. This yields the formula for $n_\mathcal{G, \text{sparse}}$.
\end{proof}

For proving Theorem \ref{thm:rank_condition}, we use the following Lemma. Throughout the proofs of the lemma and the theorem, we write $\mathbb{R}^{J} = \text{span}\{e_i : i \in I\} \subseteq \R^m $ for $I \subseteq [m]$, and we assume all operations between two vectors take place pointwise, that is, $xy, x/y, x^k$ represent pointwise product, division, and power, respectively.
\begin{lemma}\label{lem:lemma_for_thm_rank_cond} Let $\mathcal{X}, \mathcal{Y} \subseteq \R^m$ two vector spaces satisfying the following: 
\begin{enumerate}
    \item[a)] The spaces are generic in the sense that for both, $\mathcal{Z} = \mathcal{X}, \mathcal{Y}$, \[\dim(\mathcal{Z} + \R^I) = \max( \dim(\mathcal{Z}) + |I|, m)\]for all index sets $I\subseteq [m]$.
    \item[b)] $\dim(\mathcal{X}) + \dim(\mathcal{Y}) \leq m$.
\end{enumerate}
 Then, the set 
$W = \{w \in \R^m: \mathcal{Y} \cap w\mathcal{X} \neq \{0\}\}$
has Lebesgue measure zero.
\end{lemma}
\begin{proof} 
Our strategy involves parameterizing $W$ using a lower-dimensional subspace of $\R^{\ell+1}$. 
To this end, we  rewrite $W$ as 
\begin{align*}
  W&= \{w : y = w x \text{ for some }x \in \mathcal{X}\setminus \{0\}, y \in \mathcal{Y}\setminus \{0\}\} \end{align*}
As we seek to express $w$ in terms of $x$ and $y$ while avoiding division by zero, we distinguish which entries of $x$ are zero and decompose $W$ as
  \begin{align*}
    W  &=\underset{\substack{I \subseteq [m], \\ |I|<\max(\dim(\mathcal{X}), \dim(\mathcal{Y}))}}{\bigcup} W(I),
\end{align*}
with  \[W(I) = \left\{ w: w_{I^\comp} = y_{I^\comp}/ x_{I^\comp} \text{ for } x \in \mathcal{X}, y \in \mathcal{Y}  \text{ with } x_{I} \equiv y_{I} \equiv 0, x_{j} \neq 0  \text{ for all } j \in I^\comp\right\}.\]
When showing that each $W(I)$ has measure zero, we can, without loss of generality, restrict to $I=(1, \dots, i)$. Denote by $d_1, d_2$ the dimensions of $\mathcal{X} \cap \mathbb{R}^{I}, \mathcal{Y} \cap \mathbb{R}^{I}$, respectively, and let $X, Y$ be matrices whose columns form a basis of the two spaces. Then, $W(I)$ is the image of the map
\[f:\mathbb{R}^{i} \times  U  \times \mathbb{R}^{d_2} \to \mathbb{R}^{\ell+1}, (\alpha, \beta, \gamma) \mapsto \begin{pmatrix}
  \alpha \\
  (Y\gamma)_{I^\comp}/ (X\beta )_{I^\comp}
\end{pmatrix}
\] with $U = \{\beta: (X\beta)_j \neq 0 \text{ for all } j \in I^\comp\} \subseteq \mathbb{R}^{d_1}$. Inconveniently, this definition space might have dimension $\ell+1$. However, we can reduce it by exploiting that we can fix the scale of $\beta$ without changing the image.  We once again split $W(I)$ up to avoid division by zero, as
\[W(I) = \text{im}(f) = f(\mathcal{D}_0) \cup f(\mathcal{D}_1)\]
with $\mathcal{D}_0 =\mathbb{R}^{i} \times  \{\beta \in U: \beta_1 = 0\}  \times \mathbb{R}^{d_2}$ and $\mathcal{D}_1 = \mathbb{R}^{i} \times  \{\beta \in U: \beta_1 \neq 0\}  \times \mathbb{R}^{d_2}$. Now, 
$f(\mathcal{D}_1) = f(\tilde{\mathcal{D}}_1)$ for \[\tilde{\mathcal{D}}_1 = \mathbb{R}^{i}  \times \{\beta \in U: \beta_1 = 1 \}\times \mathbb{R}^{d_2}\]
since for every $(\alpha, \beta, \gamma) \in \mathcal{D}_1$,
\[f(\alpha, \beta, \gamma) = f(\alpha, \beta/\beta_1, \gamma/\beta_1).\]
Now the dimensions \[\dim(\mathcal{D}_0) = \dim(\tilde{\mathcal{D}}_1) = i + (d_1-1) + d_2, \]
are sufficiently small, that is, lower than $\ell+1$. To see that, we use that assumption a) yields
\[d_2 = \text{dim}(\mathcal{Y}\cap \mathbb{R}^{I^\comp})= \text{dim}(\mathcal{Y})+\text{dim}(\mathbb{R}^{I^\comp})-\text{dim}(\mathcal{Y} + \mathbb{R}^{I^\comp}) = \min(0, \dim(\mathcal{Y}) + (m-i) - m).\]
Similarly, \[d_1 = \text{dim}(\mathcal{Y}\cap \mathbb{R}^{I^\comp}) = \min(0, \dim(\mathcal{X}) + (m-i) - m).\] Summing these results up and then using assumption b), we obtain
\[d_1 + d_2 - 1 + i \leq (\dim(\mathcal{X})-i)  + (\dim(\mathcal{Y})-i) - 1 +i = \dim(\mathcal{X})  + \dim(\mathcal{Y})-i - 1  < m.\]
Since $f$ is differentiable and $\mathcal{D}_0, \tilde{\mathcal{D}}_1 \subseteq \R^{d_1 + d_2 - 1 + i}$ are open, by Sard's Theorem \cite[Chapter 6]{Lee2012}, both, $f(\mathcal{D}_0), f(\tilde{\mathcal{D}}_1) \subseteq \R^m$ have measure zero and so has $W$ as a countable union of measure zero sets.
\end{proof}
\begin{proof}\textbf{of Theorem \ref{thm:rank_condition}} 
a) The matrix $A_{1\to 2}^{(k_1, \ldots, k_2)}$ consists of cumulants $c^{(k)}_{i_1\ldots i_k}$ where at least one of the indices $i_1,\ldots, i_k$ equals one. For these cumulants, Lemma \ref{lem:parametrization} gives
\begin{align*}
c^{(k)}_{i_1\ldots i_k} 
&= \sum_{j=1}^{2+\ell} \cumeta{k}{j} (b_{i_1 j} \cdots b_{i_k j})\\
&= \sum_{j=1}^{2} \cum^{(k)}(\epsilon_j) (b_{i_1 j} \cdots b_{i_k j}) + \sum_{j=1}^{\ell} \cum^{(k)}(L_j) (b_{i_1 ,j+2} \cdots b_{i_k,j+2})\\
&=  \cum^{(k)}(\epsilon_1) (b_{i_1 1} \cdots b_{i_k 1}) + \sum_{j=1}^{\ell} \cum^{(k)}(L_j) (b_{i_1 ,j+2} \cdots b_{i_k ,j+2})\\
&=  \cum^{(k)}(\epsilon_1) b_{21}^{\#\{i_j=2\}} + \sum_{j=1}^{\ell} \cum^{(k)}(L_j) b_{2,j+2}^{\#\{i_j=2\}}
\end{align*}
where the penultimate equality follows since at least one index $i_r=1$ and the corresponding $b_{i_r2}=0$, resulting in $b_{i_12}\cdots b_{i_k2}=0$.
Therefore, $A^{(k_1, \ldots, k_2)}$ can be written as $A^{(k_1, \ldots, k_2)} = MN$ for
\begin{equation}\label{eq:matrix_M}
 M = \begin{pmatrix}
  \cumeta{k_1}{1} & \cumeta{k_1}{3} & \cdots & \cumeta{k_1}{\ell+2}\\
  \cumeta{k_1+1}{1} & \cumeta{k_1+1}{3} & \cdots & \cumeta{k_1+1}{\ell+2}\\
  b_{21} \cumeta{k_1+1}{1} & b_{23}\cumeta{k_1+1}{3}  & \cdots &  b_{2, 2+\ell} \cumeta{k_1+1}{\ell+2}\\
  \vdots &  \vdots & \ddots & \vdots  \\
  \cumeta{k_2}{1} &  \cumeta{k_2}{3} & \cdots & \cumeta{k_2}{\ell+2} \\
  \vdots &  \vdots & \ddots & \vdots  \\
  b_{21}^{(k_2-k_1+1)} \cumeta{k_2}{1} & b_{23}^{(k_2-k_1+1)} \cumeta{k_2}{3} & \cdots &  b_{2,2+\ell}^{(k_2-k_1+1)} \cumeta{k_2}{\ell+2} \\
  \end{pmatrix} \in \mathbb{R}^{1 + \dots + (k_2-k_1+1) \times \ell+1},
\end{equation} 
and \begin{equation*}
  N = \begin{pmatrix} 
    1 & b_{21}  & b_{21}^2 & \cdots & b_{21}^{\ell} \\
    1 & b_{23}  & b_{23}^2 & \cdots & b_{23}^{\ell} \\
    \vdots & \vdots  & \vdots & \ddots & \vdots \\
    1 & b_{2,\ell+1}  & b_{2,\ell+1}^2 & \cdots & b_{2,\ell+1}^{\ell} \\
    \end{pmatrix} \in \mathbb{R}^{\ell+1 \times \ell+1}.
\end{equation*}
Since $N$ is invertible, the rank of $A^{(k_1, \ldots, k_2)}$ coincides with the rank of $M$. The matrix $M$ has only $\ell+1$ columns, so its rank is at most $\ell+1$. 

The statement that $M$ has precisely rank $\min(\ell+1, m)$ is trivially fulfilled if $\min(\ell+1, m) =m$. So, we assume $\ell+1>m$ and show that the first $\ell+1$ rows of $M$ are linearly independent. 
To simplify notation we write $w_0, \dots, w_q = \omega^{(k_1)}, \dots, \omega^{(k_1 + q)}$ where $k_1 + q$ is the largest order appearing within the first $\ell+1$ rows of $M$. With this notation, the first $\ell+1$ rows of the matrix $M$ read as
\[w_0, w_1, w_1 b, w_2, w_2 b, w_2 b^2, w_3, \dots, w_q b^{\nu}\]
for some $\nu \leq q$. The fact that each set
\[w_j \mathcal{B}_{j} = \{w_j, w_j b, \dots, w_j b^j\}\]  is linearly independent motivates a proof by induction with the induction assumption that the set $w_0 \mathcal{B}_{0} \cup \dots \cup w_{j-1} \mathcal{B}_{j-1}$ is linearly independent. The induction base that $w_0 \mathcal{B}_{0}$ is independent is trivially fulfilled. For general $j$, we denote $\mathcal{Y}_{j-1} = \text{span}(w_0 \mathcal{B}_0 \cup \dots \cup w_{j-1} \mathcal{B}_{j-1})$, $\mathcal{X}_j = \text{span}(\mathcal{B}_j)$, and aim to show that the sum
\[\mathcal{Y}_{j-1} + w_j\mathcal{X}_j\]
is direct by employing Lemma \ref{lem:lemma_for_thm_rank_cond}. To verify its assumption b) for $\mathcal{Y}_{j-1}$, let $i=\max(\ell+1-\dim(\mathcal{Y}_{j-1}), |I|)$ and denote by $\overline{v}$ the vector with all entries with index in $[i]$ set to zero. We obtain
\begin{align*}
\mathcal{Y}_{j-1} + \mathbb{R}^{I^\comp} &\supseteq \text{span}\{e_1, \dots, e_{i}, w_0, w_1 , w_1 b, w_2 , w_2  b, w_2 b^2, w_3, \dots, w_{j-1} b^{j-1}\} \\
&= \text{span}\{e_1, \dots, e_{i}, \overline{w_0}, \overline{w_1}, \overline{w_1 b}, \overline{w_1}, \overline{w_j b}, \overline{w_j b^2}, \overline{w_j}, \dots, \overline{w_j b^{j-1}}\}\\
&= \text{span}\{e_1, \dots, e_{i}\} \oplus \text{span}\{\}\\
&= \text{span}\{e_1, \dots, e_{i}\} \oplus \text{span}\{\overline{w_j }, \overline{w_{1}} \overline{b}, \overline{w_{2}}, \overline{w_{2}} \overline{b}, \overline{w_{2}} (\overline{b})^2, \overline{w_{3}}, \dots, \overline{w_{j-1}} (\overline{b})^{j-1}\}
\end{align*} 
In the second span, we can regard $\bar{b}$ and each $\bar{w_\iota}$ as an element of $\mathbb{R}^{\ell+1-i}$  by omitting all entries that were set to zero. Note that the induction assumption holds for arbitrary $\ell'$ as long as $\ell'+1 \geq |w_0 \mathcal{B}_{0} \cup \dots \cup w_{j-1} \mathcal{B}_{j-1}|$. However, we have chosen $i$ in a way that $\ell' = \ell+1-i$ satisfies this condition. So, the induction assumption yields independence of the vectors in the second span. Therefore, \[\dim(\mathcal{Y}_{j-1} + \mathbb{R}^{I}) \geq \dim(\mathcal{Y}_{j-1}) + i = \max(\ell+1, \dim(\mathcal{Y})+|I|).\] Equality follows since the dimension of a sum of two vector spaces is always bounded by the sum of their single dimension.
Similarly, a) is fulfilled for $\mathcal{X}$. Condition b) holds as we consider the first $\ell+1$ rows of $M$.
Hence, the Lemma yields that the set of $w_j$ such that the sum $\mathcal{Y}_{j-1} + w_j\mathcal{X}_j$ is not direct is of measure $0$, or equivalently, that the sum is direct for generic $w_j$. Combining that the sum is direct, 
$w_0 \mathcal{B}_{0} \cup \dots \cup w_{j-1} \mathcal{B}_{j-1}$ is linear independent by induction assumption, and $w_j\mathcal{B}_j$ is linear independent, concludes the induction proof.

Noting that all the above arguments remain valid if $b_{21}=0$ and the coefficients are generic otherwise, and that the proofs for b) and c)  work similarly, completes the overall proof. 
\end{proof}

\begin{proof}\textbf{of Lemma \ref{lemma:marg_distr}} 
We denote the set of blocked ancestors of $w$ given $v$ as
 $$\text{bl}_v(w) = \{Z_j \in \text{an}(w): \text{all directed paths from $Z_j$ to $w$ contain $v$}\}$$ and use the shorthand $\text{an}(v, w)$ for $\text{an}(v) \cap \text{an}(w)$.
First, assume that there is  a path from $v$ to $w$. 
Then, $\text{an}(v)$ is the the disjoint union $\text{an}(v,w) \setminus \text{bl}_v(w)$ and $\text{bl}_v(w)$, similarly $\text{an}(w) = (\text{an}(v,w) \setminus \text{bl}_v(w)) \cup (\text{bl}_v(w)) \cup (\text{an}(w)\setminus\text{an}(v))$. Thus, we can write $(X_v, X_w)$ as
\begin{align*}
X_v &= \sum_{Z_j \in \text{an}(v)} b_{vj} \eta_j = \sum_{Z_j \in \text{an}(v, w) \setminus \text{bl}_v(w)} b_{vj} \eta_j + \sum_{Z_j \in \text{bl}_v(w)} b_{vj} \eta_j\\
X_w &=  \sum_{Z_j \in \text{an}(w)} b_{wj} \eta_j = \sum_{Z_j \in \text{an}(v, w)\setminus \text{bl}_v(w)} b_{wj} \eta_j + \sum_{Z_j \in \text{an}(w)\setminus\text{an}(v)} b_{wj} \eta_j + \sum_{Z_j \in \text{bl}_v(w)} b_{wj} \eta_j\\
&=  \sum_{Z_j \in \text{an}(v, w)\setminus \text{bl}_v(w)} b_{wj} \eta_j + \sum_{Z_j \in \text{an}(w)\setminus\text{an}(v)} b_{wj} \eta_j + \sum_{Z_j \in \text{bl}_v(w)} b_{wv}b_{vj} \eta_j\\
&=  \sum_{Z_j \in \text{an}(v, w)\setminus \text{bl}_v(w)} b_{wj} \eta_j + \sum_{Z_j \in \text{an}(w)\setminus\text{an}(v)} b_{wj} \eta_j + b_{wv}\left(\sum_{Z_j \in \text{bl}_v(w)} b_{vj} \eta_j\right).
\end{align*}
Now, for each $Z_j \in (\text{an}(v,w) \setminus \text{bl}_v(w)) \setminus \text{conf}(v,w)$ there exists some $\text{sw}(Z_j) \in \text{conf}(v,w)$ through which all paths from $Z_j$ to $v$ or $w$ run. Let 
$$L_i' = \eta_i  + \sum_{j: \text{sw}(Z_j) = Z_i} b_{ij}\eta_j \quad \text{ for } Z_i \in \text{conf}(v,w).$$
Then,
\begin{equation*}
    \sum_{Z_j \in \text{an}(v,w) \setminus \text{bl}_v(w)} b_{vj} \eta_j = \sum_{Z_j \in \text{conf}(v, w)} b_{vj} L_j'
\end{equation*}  
Hence, choosing $\ell' = |\text{conf}(v,w)|$,  $\{L_1', \dots, L_\ell'\}$ as mentioned, $\epsilon_1' = \sum_{j \in \text{an}(v)\setminus\text{conf}(v,w)} b_{vj} \eta_j$, 
 $\epsilon_2' = \sum_{j \in \text{an}(w)\setminus\text{an}(v)} b_{vj} \eta_j$,
 $b_{21}' = b_{vw}$, and
 $(b_{2,1+2}', \dots,b_{2,\ell+2}') =(b_{w,j}$, $Z_j \in \text{conf}(v,w))$,
the structural equations postulated by $\mathcal{M}_{2, \ell'}$ are fulfilled. If there is no path from $v$ to $w$, the proof works similarly. The only difference is that $\text{bl}_v(w)$ needs to be replaced by $\text{an}(v)\setminus \text{an}(w)$.

 If $v$ is a source, $\text{an}(v, w)\setminus \text{bl}_v(w) = \text{conf}(v,w)$, which results in the parameters given in the Lemma.
\end{proof}

\vskip 0.2in
\bibliography{references}

\begin{thebibliography}{19}
\providecommand{\natexlab}[1]{#1}
\providecommand{\url}[1]{\texttt{#1}}
\expandafter\ifx\csname urlstyle\endcsname\relax
  \providecommand{\doi}[1]{doi: #1}\else
  \providecommand{\doi}{doi: \begingroup \urlstyle{rm}\Url}\fi

\bibitem[Adams et~al.(2021)Adams, Hansen, and Zhang]{Adams2021}
Jeffrey Adams, Niels Hansen, and Kun Zhang.
\newblock Identification of partially observed linear causal models: Graphical
  conditions for the non-{G}aussian and heterogeneous cases.
\newblock \emph{Advances in Neural Information Processing Systems},
  34:\penalty0 22822--22833, 2021.

\bibitem[Auddy and Yuan(2023)]{Auddy2023}
Arnab Auddy and Ming Yuan.
\newblock Large dimensional independent component analysis: Statistical
  optimality and computational tractability, 2023.
\newblock arXiv preprint.

\bibitem[Barber et~al.(2022)Barber, Drton, Sturma, and Weihs]{BarberDSW2022}
Rina~Foygel Barber, Mathias Drton, Nils Sturma, and Luca Weihs.
\newblock Half-trek criterion for identifiability of latent variable models.
\newblock \emph{Ann. Statist.}, 50\penalty0 (6):\penalty0 3174--3196, 2022.

\bibitem[Cai et~al.(2023)Cai, Huang, Chen, Hao, and Zhang]{Cai2023}
Ruichu Cai, Zhiyi Huang, Wei Chen, Zhifeng Hao, and Kun Zhang.
\newblock Causal discovery with latent confounders based on higher-order
  cumulants.
\newblock \emph{Proceedings of the 40th International Conference on Machine
  Learning}, PMLR 202:\penalty0 3380--3407, 2023.

\bibitem[Chen et~al.(2024)Chen, Huang, Cai, Hao, and Zhang]{Chen2024}
Wei Chen, Zhiyi Huang, Ruichu Cai, Zhifeng Hao, and Kun Zhang.
\newblock Identification of causal structure with latent variables based on
  higher order cumulants.
\newblock \emph{Proceedings of the AAAI Conference on Artificial Intelligence},
  38\penalty0 (18):\penalty0 20353--20361, 2024.

\bibitem[Comon and Jutten(2010)]{Comon_2010}
Pierre Comon and Christian Jutten.
\newblock \emph{Handbook of Blind Source Separation}.
\newblock Academic Press, Oxford, 2010.

\bibitem[Entner and Hoyer(2010)]{Entner2010}
Doris Entner and Patrik~O. Hoyer.
\newblock Discovering unconfounded causal relationships using linear
  non-{G}aussian models.
\newblock In \emph{JSAI-isAI Workshops}, 2010.

\bibitem[Hoyer et~al.(2008)Hoyer, Shimizu, Kerminen, and Palviainen]{Hoyer2008}
Patrik~O. Hoyer, Shohei Shimizu, Antti~J. Kerminen, and Markus Palviainen.
\newblock Estimation of causal effects using linear non-{G}aussian causal
  models with hidden variables.
\newblock \emph{Internat. J. Approx. Reason.}, 49\penalty0 (2):\penalty0
  362--378, 2008.

\bibitem[Lee(2012)]{Lee2012}
John~M. Lee.
\newblock \emph{Introduction to Smooth Manifolds}.
\newblock Graduate Texts in Mathematics. Springer New York, NY, 2nd edition,
  2012.

\bibitem[Maathuis et~al.(2019)Maathuis, Drton, Lauritzen, and
  Wainwright]{handbook}
Marloes Maathuis, Mathias Drton, Steffen Lauritzen, and Martin Wainwright,
  editors.
\newblock \emph{Handbook of graphical models}.
\newblock Chapman \& Hall/CRC Handbooks of Modern Statistical Methods. CRC
  Press, Boca Raton, FL, 2019.

\bibitem[Maeda and Shimizu(2020)]{Maeda2020}
Takashi~Nicholas Maeda and Shohei Shimizu.
\newblock {RCD}: Repetitive causal discovery of linear non-{G}aussian acyclic
  models with latent confounders.
\newblock \emph{Proceedings of the Twenty Third International Conference on
  Artificial Intelligence and Statistics}, PMLR 108:\penalty0 735--745, 2020.

\bibitem[Salehkaleybar et~al.(2020)Salehkaleybar, Ghassami, Kiyavash, and
  Zhang]{Salehkaleybar2020}
Saber Salehkaleybar, AmirEmad Ghassami, Negar Kiyavash, and Kun Zhang.
\newblock Learning linear non-{G}aussian causal models in the presence of
  latent variables.
\newblock \emph{J. Mach. Learn. Res.}, 21\penalty0 (39):\penalty0 1--24, 2020.

\bibitem[Shimizu(2022)]{Shimizu2022}
Shohei Shimizu.
\newblock \emph{Statistical causal discovery: {L}i{NGAM} approach}.
\newblock SpringerBriefs in Statistics. Springer Japan, Tokyo, 2022.

\bibitem[Shimizu et~al.(2006)Shimizu, Hoyer, Hyv\"{a}rinen, and
  Kerminen]{Shimizu2006}
Shohei Shimizu, Patrik~O. Hoyer, Aapo Hyv\"{a}rinen, and Antti Kerminen.
\newblock A linear non-{G}aussian acyclic model for causal discovery.
\newblock \emph{J. Mach. Learn. Res.}, 7\penalty0 (72):\penalty0 2003--2030,
  2006.

\bibitem[Shimizu et~al.(2011)Shimizu, Inazumi, Sogawa, Hyv\"{a}rinen, Kawahara,
  Washio, Hoyer, and Bollen]{Shimizu2011}
Shohei Shimizu, Takanori Inazumi, Yasuhiro Sogawa, Aapo Hyv\"{a}rinen,
  Yoshinobu Kawahara, Takashi Washio, Patrik~O. Hoyer, and Kenneth Bollen.
\newblock Direct{L}i{NGAM}: a direct method for learning a linear
  non-{G}aussian structural equation model.
\newblock \emph{J. Mach. Learn. Res.}, 12\penalty0 (33):\penalty0 1225--1248,
  2011.

\bibitem[Tashiro et~al.(2014)Tashiro, Shimizu, Hyvärinen, and
  Washio]{Tashiro2014}
Tatsuya Tashiro, Shohei Shimizu, Aapo Hyvärinen, and Takashi Washio.
\newblock {ParceLiNGAM}: A causal ordering method robust against latent
  confounders.
\newblock \emph{Neural Computation}, 26\penalty0 (1):\penalty0 57--83, 2014.

\bibitem[Tramontano et~al.(2024)Tramontano, Kivva, Salehkaleybar, Drton, and
  Kiyavash]{Tramontano2024}
Daniele Tramontano, Yaroslav Kivva, Saber Salehkaleybar, Mathias Drton, and
  Negar Kiyavash.
\newblock Causal effect identification in {LiNGAM} models with latent
  confounders, 2024.
\newblock arXiv preprint.

\bibitem[Wang and Drton(2020)]{WangD2020}
Y.~Samuel Wang and Mathias Drton.
\newblock High-dimensional causal discovery under non-{G}aussianity.
\newblock \emph{Biometrika}, 107\penalty0 (1):\penalty0 41--59, 2020.

\bibitem[Wang and Drton(2023)]{Wang2023}
Y.~Samuel Wang and Mathias Drton.
\newblock Causal discovery with unobserved confounding and non-{G}aussian data.
\newblock \emph{J. Mach. Learn. Res.}, 24\penalty0 (271):\penalty0 1--61, 2023.

\end{thebibliography}

\end{document}